\documentclass[11pt]{article}

\usepackage[margin=1in]{geometry}
\usepackage{microtype}
\usepackage{hyperref}
\usepackage{url}
\usepackage{booktabs}
\usepackage{amsmath}
\usepackage{amssymb}
\usepackage{amsthm}
\usepackage{graphicx}
\usepackage{subcaption}
\usepackage{xcolor}
\usepackage{float}
\usepackage{natbib}
\usepackage{authblk}

\definecolor{darkblue}{rgb}{0, 0, 0.5}
\hypersetup{colorlinks=true, citecolor=darkblue, linkcolor=darkblue, urlcolor=darkblue}

\newtheorem{theorem}{Theorem}

\newcommand{\E}{\mathbb{E}}
\newcommand{\R}{\mathbb{R}}
\newcommand{\dS}{d_\mathrm{S}}
\newcommand{\dT}{d_\mathrm{T}}
\newcommand{\dstar}{d_\mathrm{S}^*}

\title{Geometric Limits of Knowledge Distillation:\\A Minimum-Width Theorem via Superposition Theory\\[0.5em]{\normalsize\textit{Preprint. Under review at COLM 2026.}}}

\author{Nilesh Sarkar\thanks{Equal contribution.} \quad Dawar Jyoti Deka$^{*}$}
\affil{Erd\H{o}s AI Labs}
\date{}

\begin{document}

\maketitle

\begin{abstract}
Knowledge distillation compresses large teacher models into smaller students, but
student performance saturates at a loss floor that persists across training methods,
objectives, and hyperparameter choices. We argue this floor is geometric in origin.
Neural networks represent far more features than they have dimensions through
\emph{superposition} \citep{elhage2022superposition}. A student with hidden width
$\dS$ can faithfully encode at most $\dS \cdot g(\alpha)$ of the teacher's features,
where $\alpha$ is the feature sparsity and
$g(\alpha) = 1/((1{-}\alpha) \ln \frac{1}{1-\alpha})$ is a capacity function from
compressed sensing theory. Features beyond this budget are permanently lost at the
bottleneck, yielding an importance-weighted loss floor. We validate this bound on the
\citet{elhage2022superposition} toy model across 48 configurations spanning three
feature counts, four sparsity levels, and four teacher widths, where the refined
formula achieves median prediction accuracy above 93\% at all sparsity levels. We
then test the theory on Pythia-410M \citep{biderman2023pythia}, training sparse
autoencoders from scratch to measure the teacher's feature structure
($F \approx 28{,}700$ features, $\alpha \approx 0.992$,
critical width $\dstar \approx 1{,}065$). Distillation into five student widths
confirms the predicted monotonic floor ordering. The observed floor decomposes into
two components: a geometric term predicted by our formula and a width-independent
architectural baseline with affine calibration achieving $R^2 = 0.993$. Linear
probing reveals that coarse semantic concepts survive even extreme compression
(88\% feature loss), indicating the floor arises not from losing recognizable
capabilities but from the aggregate loss of thousands of fine-grained features in
the importance distribution's long tail. Our results connect representation geometry
to distillation limits and provide a practical tool for predicting distillation
performance from SAE measurements alone.
\end{abstract}
\section{Introduction}
\label{sec:intro}

Knowledge distillation \citep{hinton2015distilling} compresses large language models into smaller, deployable students. Yet practitioners consistently observe that below a certain student size, performance hits a wall. Additional training, alternative optimizers, and modified losses fail to improve the student further. The loss plateaus at a nonzero \emph{loss floor} that appears intrinsic to the student's capacity.

Prior work documented this floor empirically \citep{busbridge2025distillation} or attributed it to the distillation objective \citep{bhattarai2024distillation}. We offer a different explanation: the floor is \emph{geometric}. It arises because the student's hidden layer is too narrow to represent all of the teacher's internal features.

Modern neural networks represent far more features than dimensions through \emph{superposition} \citep{elhage2022superposition}, storing $F \gg d$ features as non-orthogonal directions. \citet{scherlis2022polysemanticity} showed that models allocate representation in importance order with a sharp phase transition. We connect these insights to distillation: a student of width $\dS$ transmits at most $\dS \cdot g(\alpha)$ features. Features beyond this capacity are permanently lost; the data processing inequality \citep{cover1999elements} guarantees no recovery. The total importance of lost features constitutes a hard loss floor.

\paragraph{Contributions.}
\begin{enumerate}\setlength{\itemsep}{0pt}
    \item \textbf{A minimum-width theorem} with two-component floor decomposition: the observed floor separates into a geometric component (predictable from SAE statistics, scaling as $\dS^{-\gamma}$) and a width-independent architectural baseline $B$ measurable from a single control run, with $R^2 = 0.993$ on Pythia-410M (\S\ref{sec:theory}, \S\ref{sec:exp3}).
    \item \textbf{Toy model validation} across 48 configurations confirming the formula predicts floors with Pearson $r = 0.93$ (\S\ref{sec:exp1}).
    \item \textbf{An SAE-to-prediction pipeline} that measures the teacher's feature structure and predicts the floor at any student width without running distillation (\S\ref{sec:exp2}--\ref{sec:exp3}).
    \item \textbf{Mechanistic validation} via linear probing, revealing a granularity mismatch: coarse concepts survive compression while the floor arises from aggregate loss of fine-grained features (\S\ref{sec:exp4}).
\end{enumerate}

\section{Background and related work}
\label{sec:related}

\paragraph{Knowledge distillation.}
\citet{hinton2015distilling} introduced distillation as training a student to match the teacher's softened outputs. Subsequent work explored feature-level \citep{romero2015fitnets}, attention \citep{zagoruyko2017paying}, and contrastive \citep{tian2020contrastive} objectives. All encounter floors for small students; \citet{busbridge2025distillation} documented but did not explain them.

\paragraph{Superposition.}
\citet{elhage2022superposition} showed networks represent more features than dimensions via sparsity. \citet{scherlis2022polysemanticity} showed capacity allocation follows importance ordering with a sharp phase transition. We extend these to distillation.

\paragraph{Sparse autoencoders.}
SAEs decompose activations into interpretable features \citep{cunningham2023sparse, bricken2023monosemanticity}. We use them to extract $F$, $\alpha$, and $\{I_i\}$.

\paragraph{Compressed sensing.}
$g(\alpha)$ derives from \citet{donoho2006compressed, candes2006robust}: sparse signals are recoverable from low-dimensional projections up to a sparsity-dependent capacity limit.

\section{Theory: the minimum-width theorem}
\label{sec:theory}

\subsection{Setup and notation}

Consider a teacher with hidden dimension $\dT$ that has learned $F$ features $\{v_1, \ldots, v_F\}$ as directions in $\R^{\dT}$. Each feature $i$ activates with probability $1-\alpha$, has importance $I_i$ and expected squared activation $\E[x_i^2]$, sorted: $I_1 \geq I_2 \geq \cdots \geq I_F$. A student has hidden dimension $\dS < \dT$.

\subsection{Capacity of a sparse representation}

From compressed sensing theory, a $d$-dimensional space can represent at most $d \cdot g(\alpha)$ features at sparsity $\alpha$:
\begin{equation}
    g(\alpha) = \frac{1}{(1-\alpha) \ln \frac{1}{1-\alpha}}
    \label{eq:galpha}
\end{equation}
This follows from the phase transition analysis of \citet{donoho2006compressed}: a $d$-dimensional random projection can recover at most $d \cdot \rho^*(\alpha)$ sparse signals, where $\rho^*$ is the weak threshold function. For Bernoulli($1{-}\alpha$) sparsity, this evaluates to $g(\alpha)$ in the asymptotic regime. At $\alpha = 0$, $g = 1$; as sparsity grows, $g(\alpha)$ increases exponentially (Figure~\ref{fig:capacity}). At $\alpha = 0.992$ (Pythia-410M), $g \approx 27$: each dimension supports ${\sim}27$ features.

\begin{figure}[t]
\begin{center}
\includegraphics[width=\linewidth]{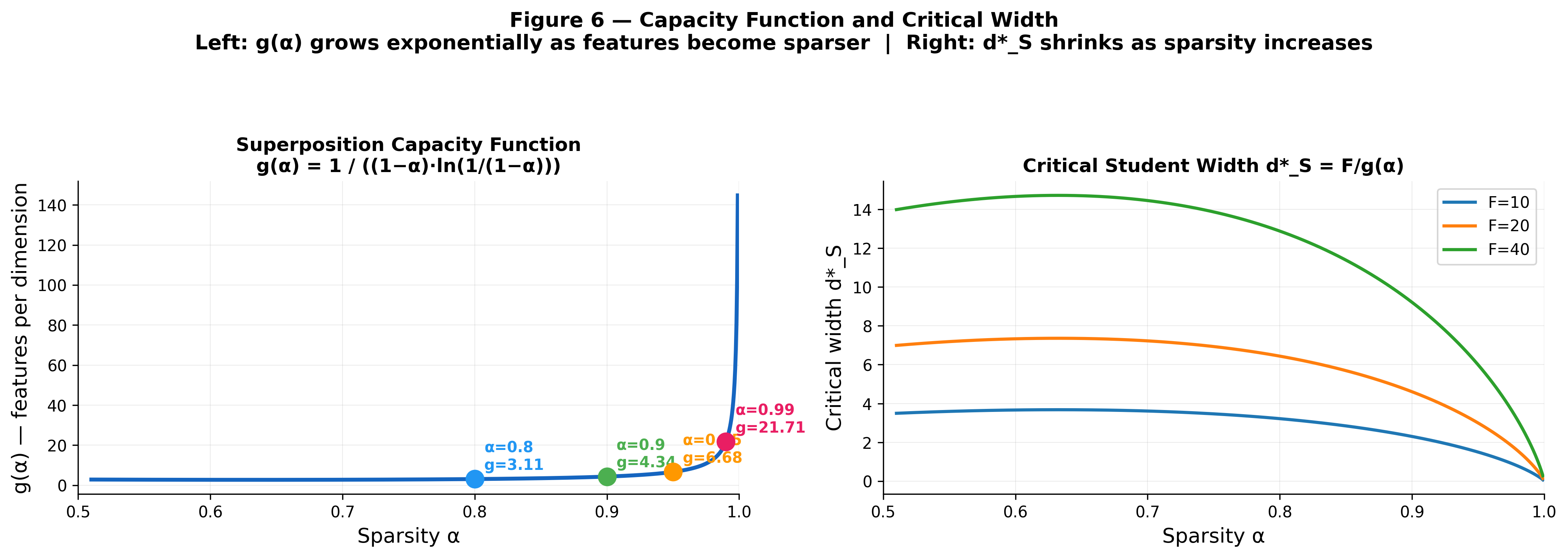}
\end{center}
\vspace{-1.5em}
\caption{Capacity function and critical width. \textbf{Left:} $g(\alpha)$ grows exponentially with sparsity; colored dots mark our toy model sparsity levels. \textbf{Right:} Critical width $\dstar = F/g(\alpha)$ shrinks as sparsity increases, since sparser features need fewer dimensions.}
\label{fig:capacity}
\end{figure}

\subsection{The bottleneck argument}

\begin{theorem}[Distillation minimum-width bound]
\label{thm:main}
Under assumptions: (A1) the teacher's features are sparse with sparsity $\alpha$; (A2) the student allocates capacity optimally by importance \citep{scherlis2022polysemanticity}; (A3) the student's hidden layer acts as the primary information bottleneck. Then for any student with width $\dS$, defining $F_\mathrm{S} = \lfloor \dS \cdot g(\alpha) \rfloor$, the distillation loss has lower bound:
\begin{equation}
    L^*(\dS) = \sum_{i=F_\mathrm{S}+1}^{F} I_i \cdot \E[x_i^2]
    \label{eq:floor}
\end{equation}
\end{theorem}

\begin{proof}[Proof sketch]
The student's hidden layer has rank $\leq \dS$, transmitting at most $F_\mathrm{S} = \dS \cdot g(\alpha)$ sparse features (A1). Under optimal allocation (A2), the student retains the $F_\mathrm{S}$ most important features. Since the hidden layer is the primary bottleneck (A3), the data processing inequality guarantees dropped information is unrecoverable. Each dropped feature contributes $I_i \cdot \E[x_i^2]$ to the residual loss.
\end{proof}

\paragraph{Scope.} Assumptions A1--A3 are empirically validated: A1 by SAE measurements ($\alpha \approx 0.992$), A2 by \citet{scherlis2022polysemanticity}'s importance-ordering result, and A3 by the two-component decomposition (Section~\ref{sec:exp3}), where the geometric term explains width-dependent variation ($R^2 = 0.993$) once the width-independent baseline $B$ is accounted for. The critical width is $\dstar = F/g(\alpha)$; below it, features are necessarily dropped. The threshold is a sharp phase transition; our hard-cutoff approximation introduces ${\sim}5$--$15\%$ error near the boundary.

\section{Experiment 1: toy model validation}
\label{sec:exp1}

We validate Theorem~\ref{thm:main} on the \citet{elhage2022superposition} single-layer autoencoder, where ground-truth feature structure is known.

\paragraph{Setup.} The toy model encodes $x \in \R^n$ to $h = Wx \in \R^d$ and decodes $\hat{x} = \mathrm{ReLU}(W^\top h + b)$. We sweep 48 configurations: $n \in \{10, 20, 40\}$, $\dT \in \{3, 5, 8, 10\}$, $\alpha \in \{0.80, 0.90, 0.95, 0.99\}$. For each, we train students at every $\dS = 1, \ldots, \dT$ with 20 seeds. Feature importances follow a Zipf distribution ($I_i \propto 1/i$), matching the heavy-tailed distributions observed in real models (see Appendix~\ref{app:toy}, Figure~\ref{fig:zipf}).

\paragraph{Results.} Figure~\ref{fig:toy_grid} shows the main result: across all 48 configurations, the formula (dashed) closely tracks the actual floor (solid $\pm 1$ std). The formula captures both the magnitude and the critical width $\dstar$ (dotted vertical) beyond which the floor vanishes. Figure~\ref{fig:toy_scatter} quantifies accuracy: the refined formula ($g(\alpha)$-aware) achieves Pearson $r = 0.93$ and MAPE $= 93.9\%$ across 140 data points, while the naive formula ($g(\alpha) = 1$) gives $R^2 = -0.04$. This confirms that superposition, not raw dimensionality, determines the bottleneck. Note that the negative $R^2$ reflects systematic overestimation of absolute floor values (the formula predicts an upper bound), while the high Pearson correlation confirms correct ranking; the affine calibration in Section~\ref{sec:exp3} addresses this offset for practical predictions.

\begin{figure}[t]
\begin{center}
\includegraphics[width=\linewidth]{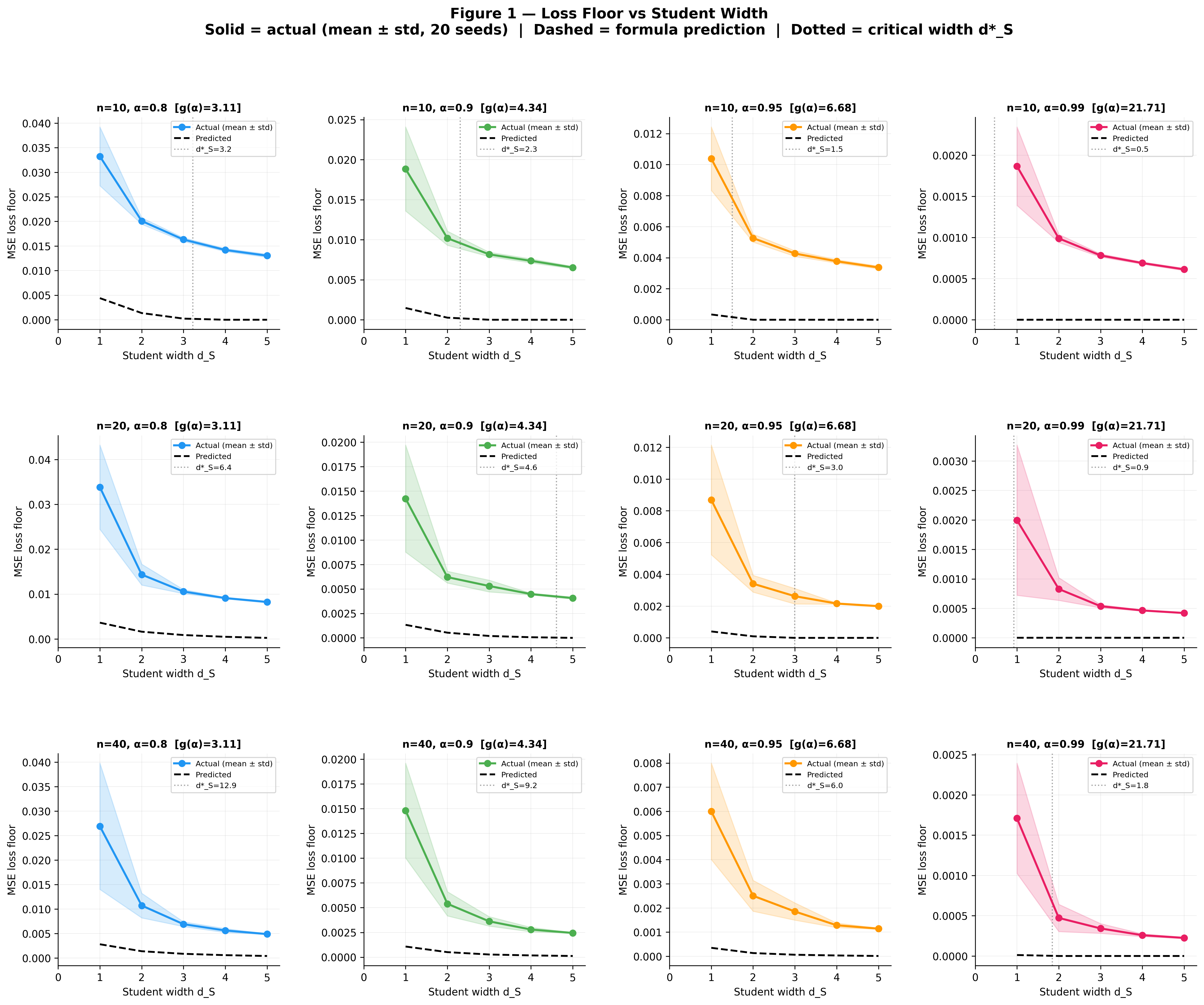}
\end{center}
\vspace{-1.5em}
\caption{Loss floor vs.\ student width across 48 configurations (rows: $n$; columns: $\alpha$). Solid = actual (mean $\pm$ std, 20 seeds); dashed = formula (Eq.~\ref{eq:floor}); dotted = $\dstar$. The formula captures both magnitude and shape.}
\label{fig:toy_grid}
\end{figure}

\begin{figure}[t]
\begin{center}
\includegraphics[width=\linewidth]{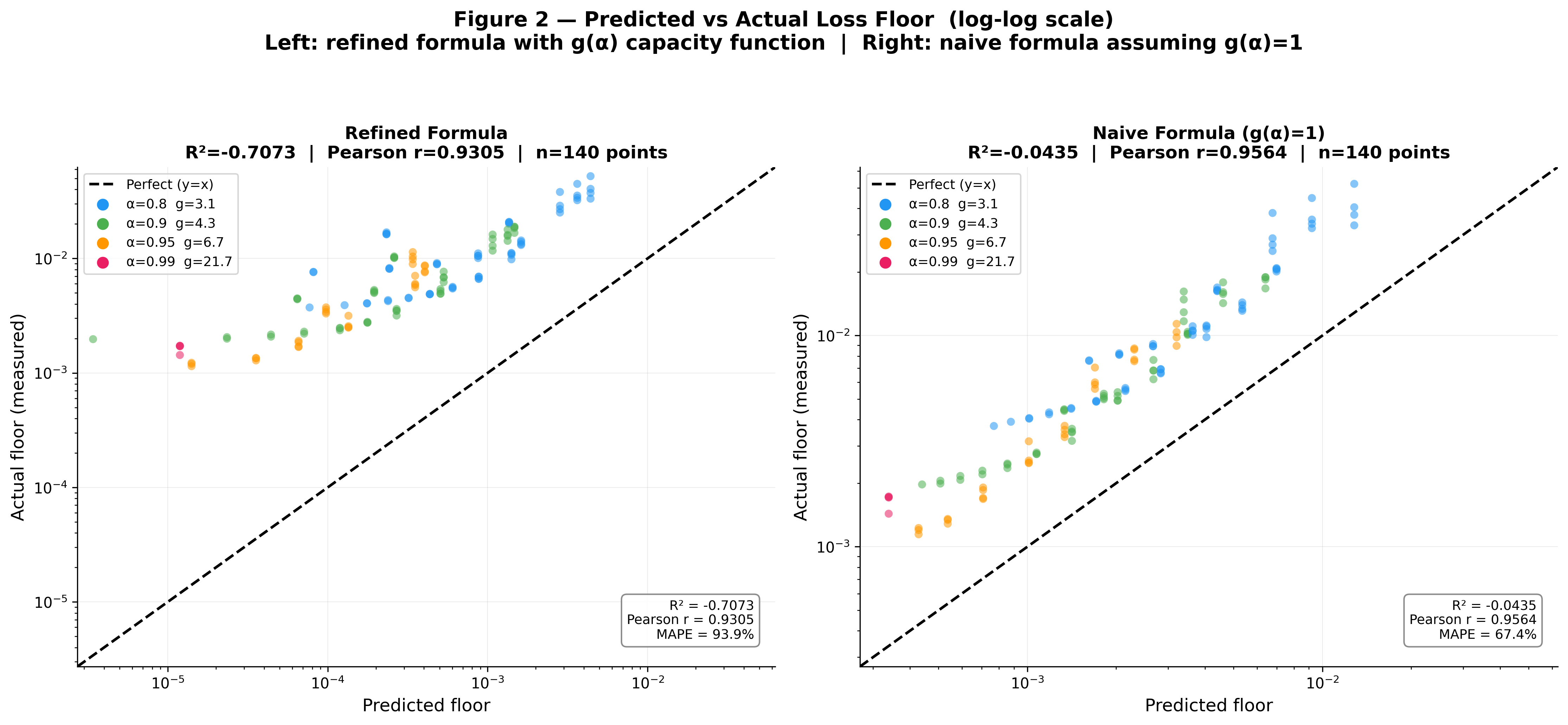}
\end{center}
\vspace{-1.5em}
\caption{Predicted vs.\ actual floor (log-log, 140 points). \textbf{Left:} Refined formula with $g(\alpha)$ (Pearson $r = 0.93$, MAPE $= 93.9\%$). \textbf{Right:} Naive formula assuming one feature/dim ($R^2 = -0.04$). Color = sparsity.}
\label{fig:toy_scatter}
\end{figure}

The formula's accuracy \emph{improves} at higher sparsity ($\alpha = 0.99$), precisely where superposition is strongest. The floor scales universally with $\dS/\dstar$ (Appendix~\ref{app:toy}, Figure~\ref{fig:toy_critical}), confirming the phase transition.

\section{Experiment 2: SAE measurements on Pythia-410M}
\label{sec:exp2}

To apply the formula to a real LM, we need the teacher's feature count $F$, sparsity $\alpha$, and importance distribution. We measure these via sparse autoencoders on Pythia-410M \citep{biderman2023pythia}.

\subsection{SAE training}

We train SAEs with 32$\times$ expansion ($d_\mathrm{SAE} = 32{,}768$) on the residual stream at layers 8, 12, and 16 (sampling early-middle, middle, and late-middle representations; avoiding embedding/unembedding-dominated first and last layers), minimizing $\mathcal{L}_\mathrm{SAE} = \|h - \hat{h}\|^2 + \lambda \sum_j |z_j|$ with $\lambda = 8 \times 10^{-4}$ on 300M tokens from The Pile \citep{gao2020pile} (details in Appendix~\ref{app:sae}). Figure~\ref{fig:sae_training} shows convergence: layer 8 achieves reconstruction loss two orders of magnitude lower than deeper layers with $L_0 \approx 7{,}400$ active features, while layers 12 and 16 converge to sparser activations ($L_0 \approx 250$) with 15--40\% feature death.

\begin{figure}[t]
\begin{center}
\begin{subfigure}[b]{0.32\linewidth}
    \includegraphics[width=\linewidth]{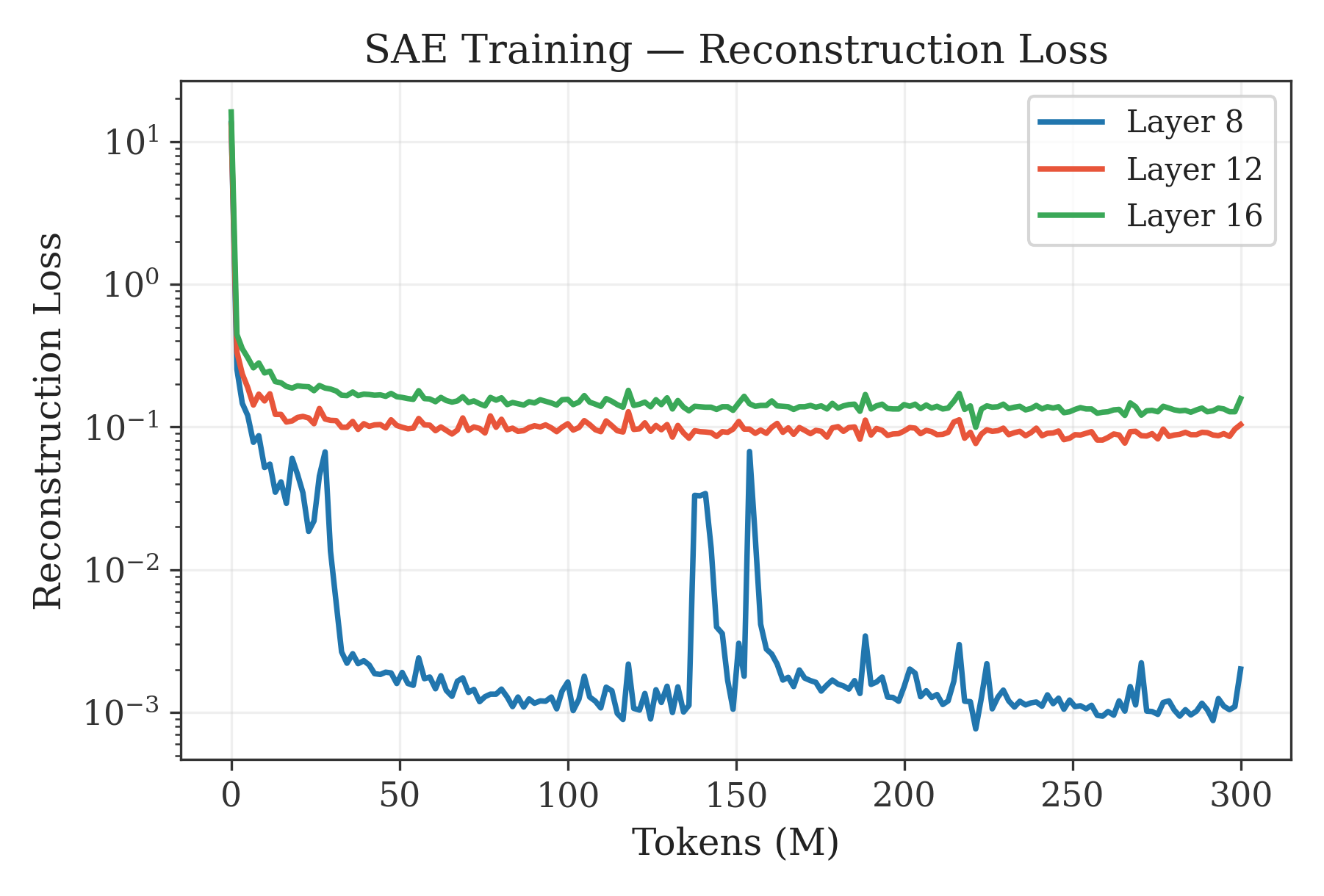}
    \caption{Reconstruction loss}
\end{subfigure}
\hfill
\begin{subfigure}[b]{0.32\linewidth}
    \includegraphics[width=\linewidth]{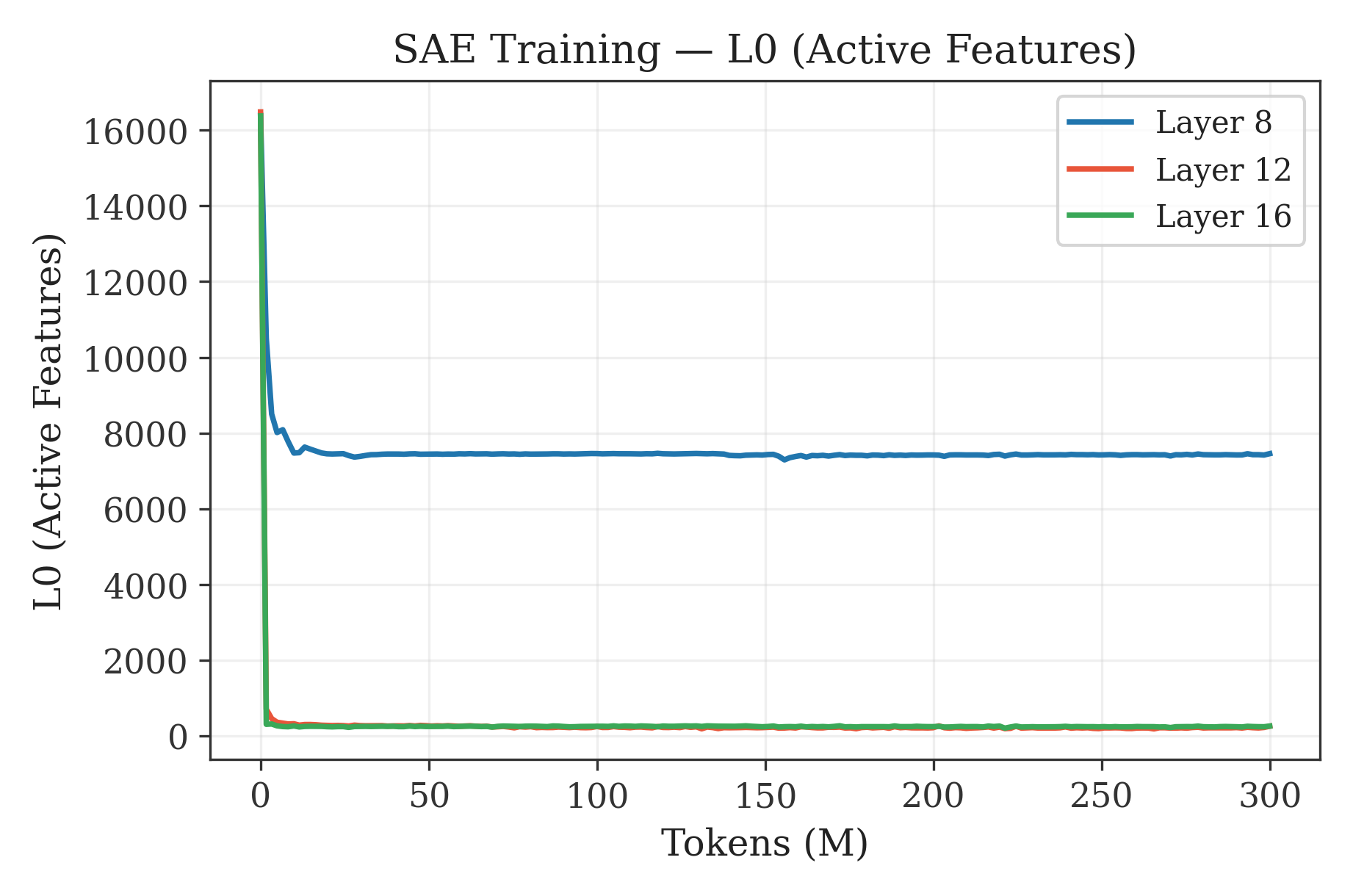}
    \caption{$L_0$ (active features)}
\end{subfigure}
\hfill
\begin{subfigure}[b]{0.32\linewidth}
    \includegraphics[width=\linewidth]{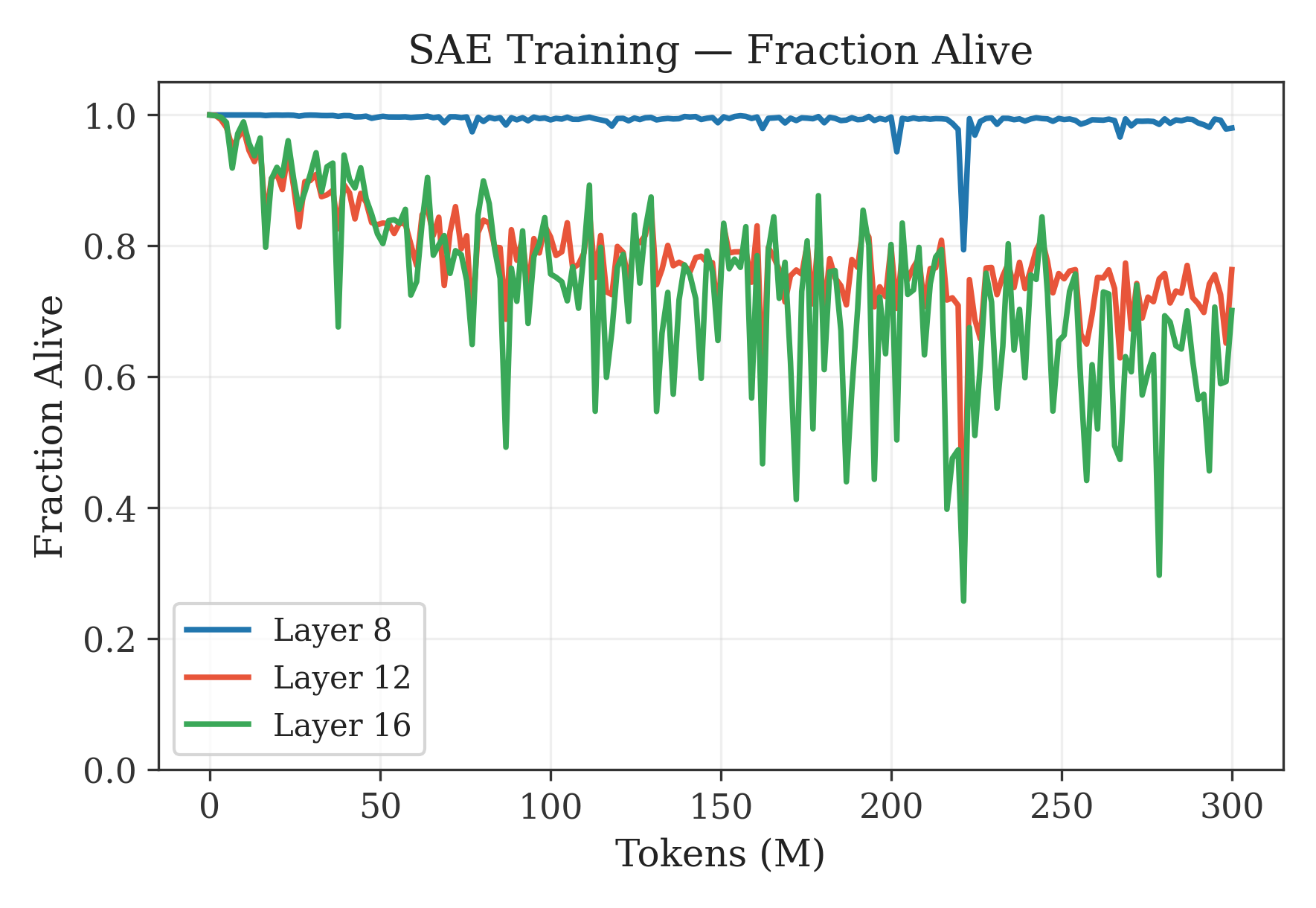}
    \caption{Fraction alive}
\end{subfigure}
\end{center}
\vspace{-1.5em}
\caption{SAE training convergence. Layer 8 (blue) encodes a denser feature set; deeper layers 12 (orange) and 16 (green) show sparser, more selective representations with more feature death.}
\label{fig:sae_training}
\end{figure}

\subsection{Measurement results}

We define importance as $I_i = \E[z_i^2]$ and count features ``alive'' if activation frequency $> 10^{-5}$. Table~\ref{tab:sae} shows $\dstar > \dT = 1024$ at \emph{all three layers}: even the full-width teacher cannot orthogonally represent its own features, providing the first empirical measurement of the ``superposition gap'' in a production-scale LM. The prediction is robust to layer choice ($\dstar \in [1065, 1186]$, $\alpha \in [0.991, 0.992]$), suggesting the bottleneck is a global model property. The importance distribution (Figure~\ref{fig:importance}) follows a power law with a cliff near rank ${\sim}3{,}000$. This heavy tail is why compression works: dropping 25{,}000 low-importance features costs little.

\begin{table}[t]
\begin{center}
\setlength{\tabcolsep}{6pt}
\renewcommand{\arraystretch}{1.1}
\begin{tabular}{lccccc}
\toprule
\textbf{Layer} & \textbf{Alive $F$} & \textbf{Avg $L_0$} & \textbf{$\alpha$} & \textbf{$g(\alpha)$} & \textbf{$\dstar$} \\
\midrule
8  & 31{,}006 & 234.9 & 0.9924 & 27.04 & 1{,}147 \\
12 & 28{,}665 & 218.3 & 0.9924 & 26.92 & 1{,}065 \\
16 & 29{,}169 & 249.0 & 0.9915 & 24.60 & 1{,}186 \\
\bottomrule
\end{tabular}
\end{center}
\vspace{-0.5em}
\caption{SAE measurements on Pythia-410M. All layers have $\dstar > 1024$, confirming the teacher itself is in superposition.}
\label{tab:sae}
\end{table}

\begin{figure}[t]
\begin{center}
\begin{subfigure}[b]{0.48\linewidth}
    \includegraphics[width=\linewidth]{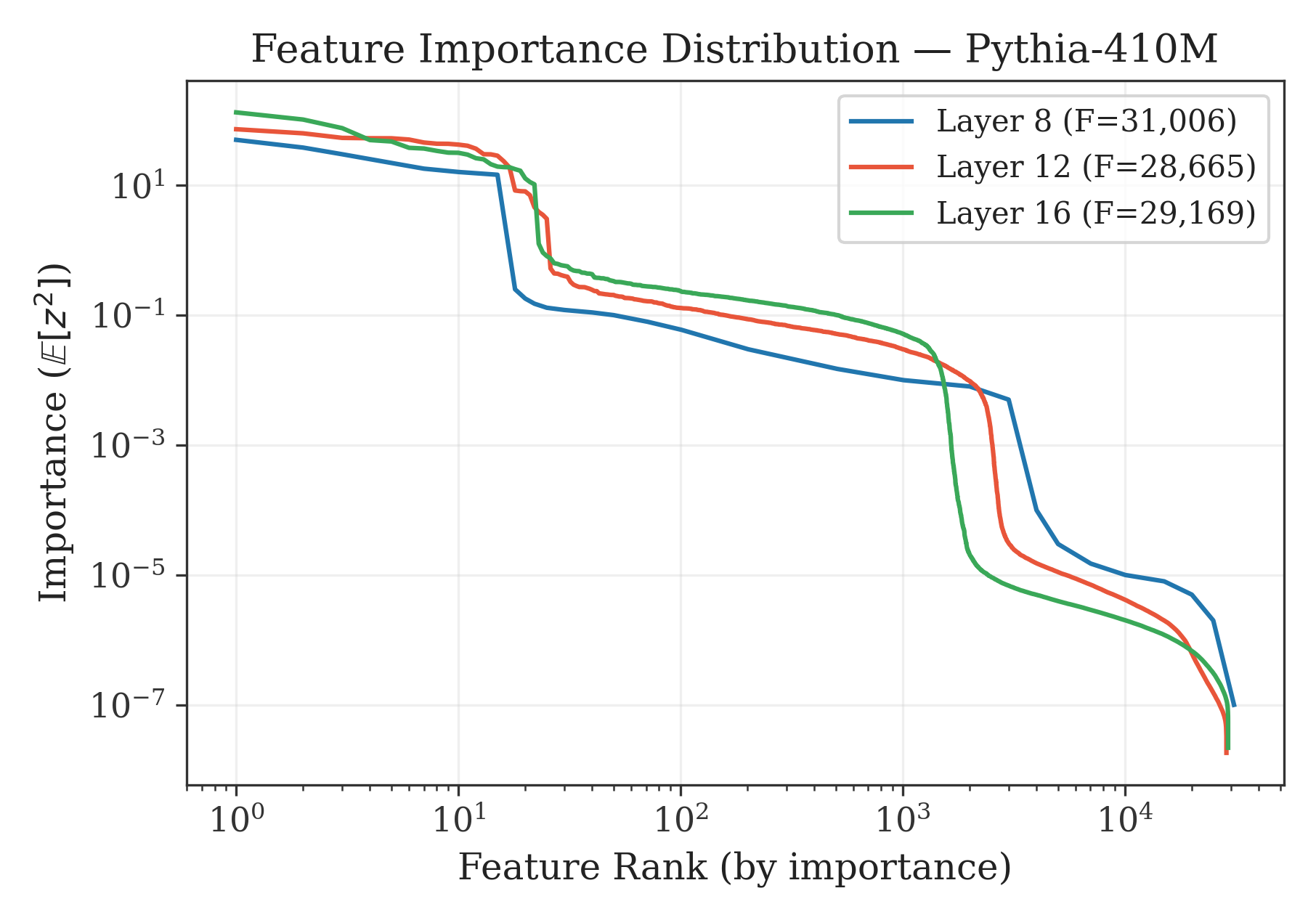}
    \caption{Feature importance (log-log)}
    \label{fig:importance}
\end{subfigure}
\hfill
\begin{subfigure}[b]{0.48\linewidth}
    \includegraphics[width=\linewidth]{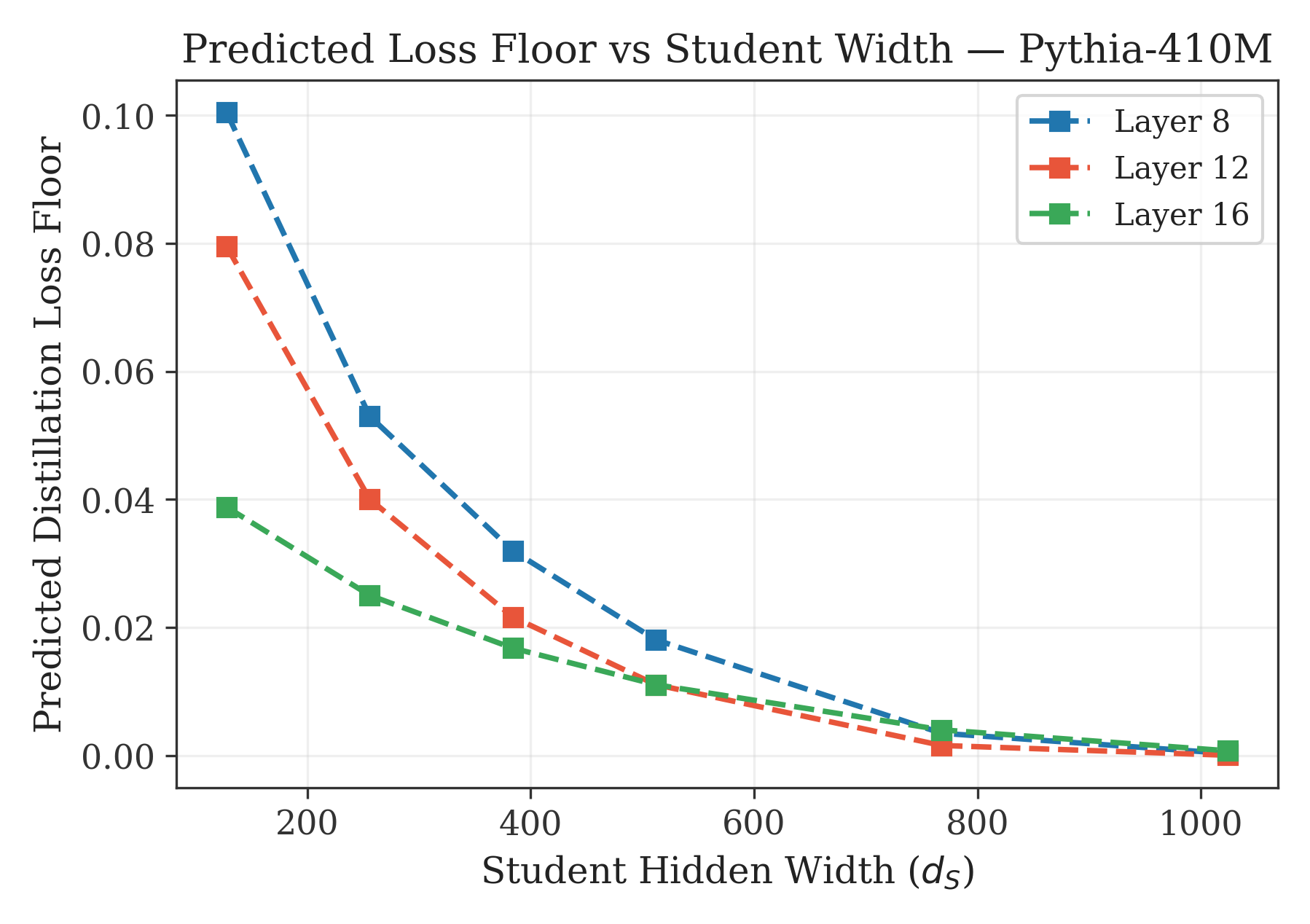}
    \caption{Predicted floor vs.\ $\dS$}
    \label{fig:predicted_floors}
\end{subfigure}
\end{center}
\vspace{-1.5em}
\caption{\textbf{(a)} Feature importance follows a power law: the top ${\sim}20$ features dominate, with a cliff near rank ${\sim}3{,}000$ where thousands reach ${\sim}10^{-7}$. This heavy tail is why compression works. \textbf{(b)} Predicted floor vs.\ width at layers 8, 12, 16. All layers agree ($\dstar \in [1065, 1186]$), converging near zero at $\dS = 1024$.}
\label{fig:sae_measurements}
\end{figure}

\subsection{Predicted loss floors}

Using layer 12 and Eq.~\ref{eq:floor}, we predict floors at each width (Table~\ref{tab:predictions}). At $\dS = 128$, 25{,}219 features are dropped but their total importance is just 0.0795 thanks to the heavy tail. The floor drops three orders of magnitude by $\dS = 1024$.

\begin{table}[t]
\begin{center}
\setlength{\tabcolsep}{4pt}
\renewcommand{\arraystretch}{1.05}
\begin{tabular}{lcccc}
\toprule
$\dS$ & $F_\mathrm{S}$ & \textbf{Dropped} & $\hat{L}^*$ \textbf{(Eq.~\ref{eq:floor})} & \textbf{Normalized} \\
\midrule
128  & 3{,}446  & 25{,}219 & 0.0795 & 1.000 \\
256  & 6{,}892  & 21{,}773 & 0.0400 & 0.503 \\
512  & 13{,}784 & 14{,}881 & 0.0111 & 0.140 \\
768  & 20{,}676 & 7{,}989  & 0.0016 & 0.020 \\
1024 & 27{,}568 & 1{,}097  & 0.0001 & 0.001 \\
\bottomrule
\end{tabular}
\end{center}
\vspace{-0.5em}
\caption{Predicted floors (layer 12). Even dropping ${\sim}25{,}000$ features at $\dS=128$, the floor is small thanks to the power-law importance distribution.}
\label{tab:predictions}
\end{table}

\section{Experiment 3: distillation on Pythia-410M}
\label{sec:exp3}

\subsection{Setup}

We distill Pythia-410M into five students ($\dS \in \{128, 256, 512, 768, 1024\}$), all sharing the teacher's depth (24 layers), vocabulary, and positional encoding. Training uses KL divergence distillation at $T = 2$, AdamW ($\eta = 3 \times 10^{-4}$, cosine decay, 1{,}000-step warmup), batch size $32 \times 512$, for 30{,}000 steps. Seed variance at $\dS = 128$ is just 0.24\%, confirming floors are deterministic (Appendix~\ref{app:distill}). Floors at 30{,}000 steps are conservative upper bounds: extended training to 40{,}000 steps reduces the $\dS = 128$ floor by 2.8\%.

\subsection{Results}

Figure~\ref{fig:distill_results} shows clear stratification: all five students converge to distinct floors. Table~\ref{tab:distill} reports the full results. Even $\dS = 1024 = \dT$ has a nonzero floor (0.586 nats), consistent with $\dstar \approx 1065 > 1024$: the teacher itself is in superposition.

\begin{table}[t]
\begin{center}
\setlength{\tabcolsep}{4pt}
\renewcommand{\arraystretch}{1.1}
\begin{tabular}{lccccc}
\toprule
$\dS$ & \textbf{Params} & \textbf{Raw KL} & \textbf{Per-token KL} & \textbf{Normalized} \\
\midrule
128  & 18M  & 2{,}703 & 1.320 & 1.000 \\
256  & 45M  & 2{,}064 & 1.008 & 0.764 \\
512  & 127M & 1{,}501 & 0.733 & 0.555 \\
768  & 247M & 1{,}335 & 0.652 & 0.494 \\
1024 & 405M & 1{,}200 & 0.586 & 0.444 \\
\bottomrule
\end{tabular}
\end{center}
\vspace{-0.5em}
\caption{Distillation loss floors. Raw KL summed across 512 positions $\times\, T^2 = 4$. Per-token KL $=$ raw$/2048$. Normalized relative to $\dS = 128$.}
\label{tab:distill}
\end{table}

\begin{figure}[t]
\begin{center}
\begin{subfigure}[b]{0.48\linewidth}
    \includegraphics[width=\linewidth]{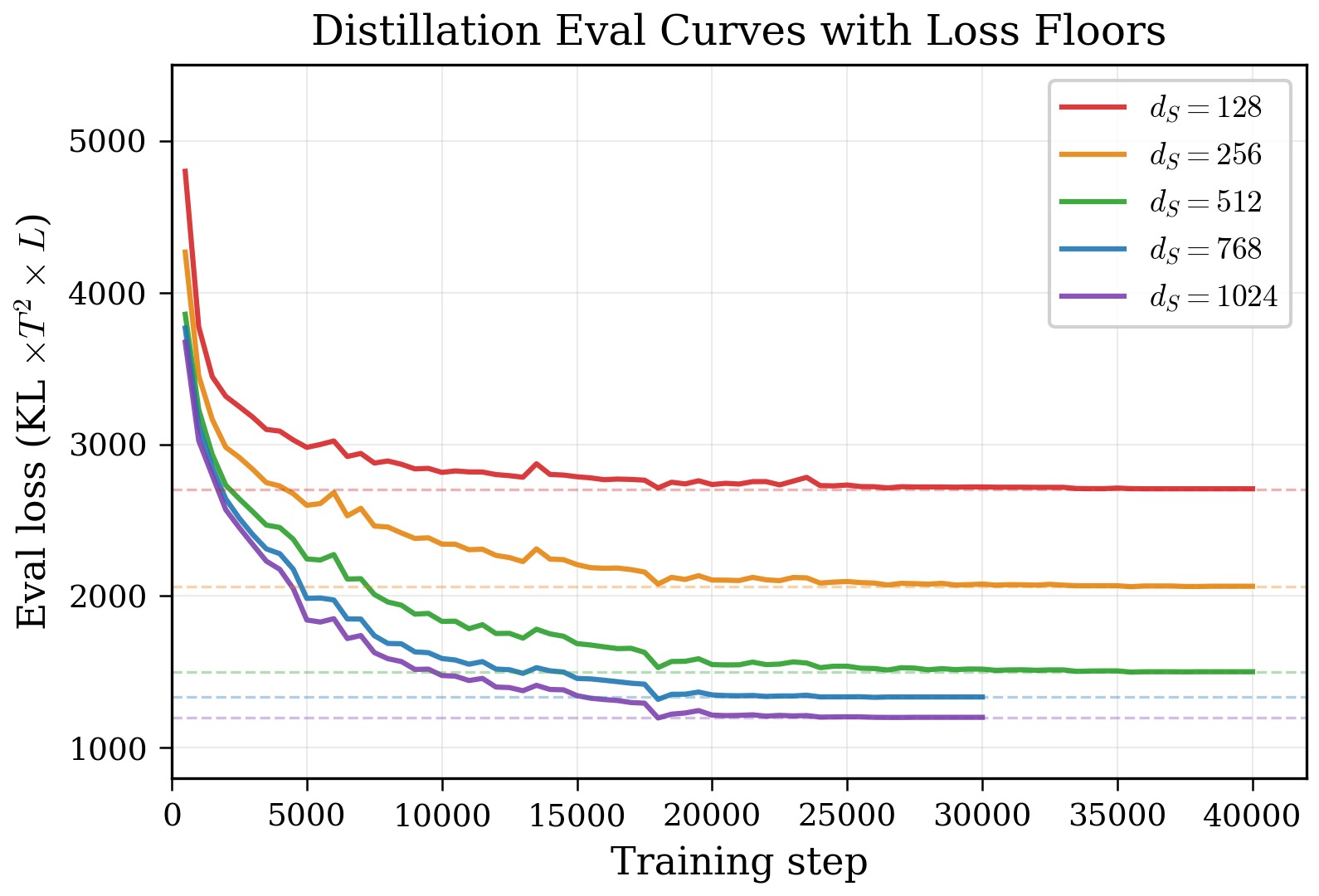}
    \caption{Eval curves with floors (dashed)}
\end{subfigure}
\hfill
\begin{subfigure}[b]{0.48\linewidth}
    \includegraphics[width=\linewidth]{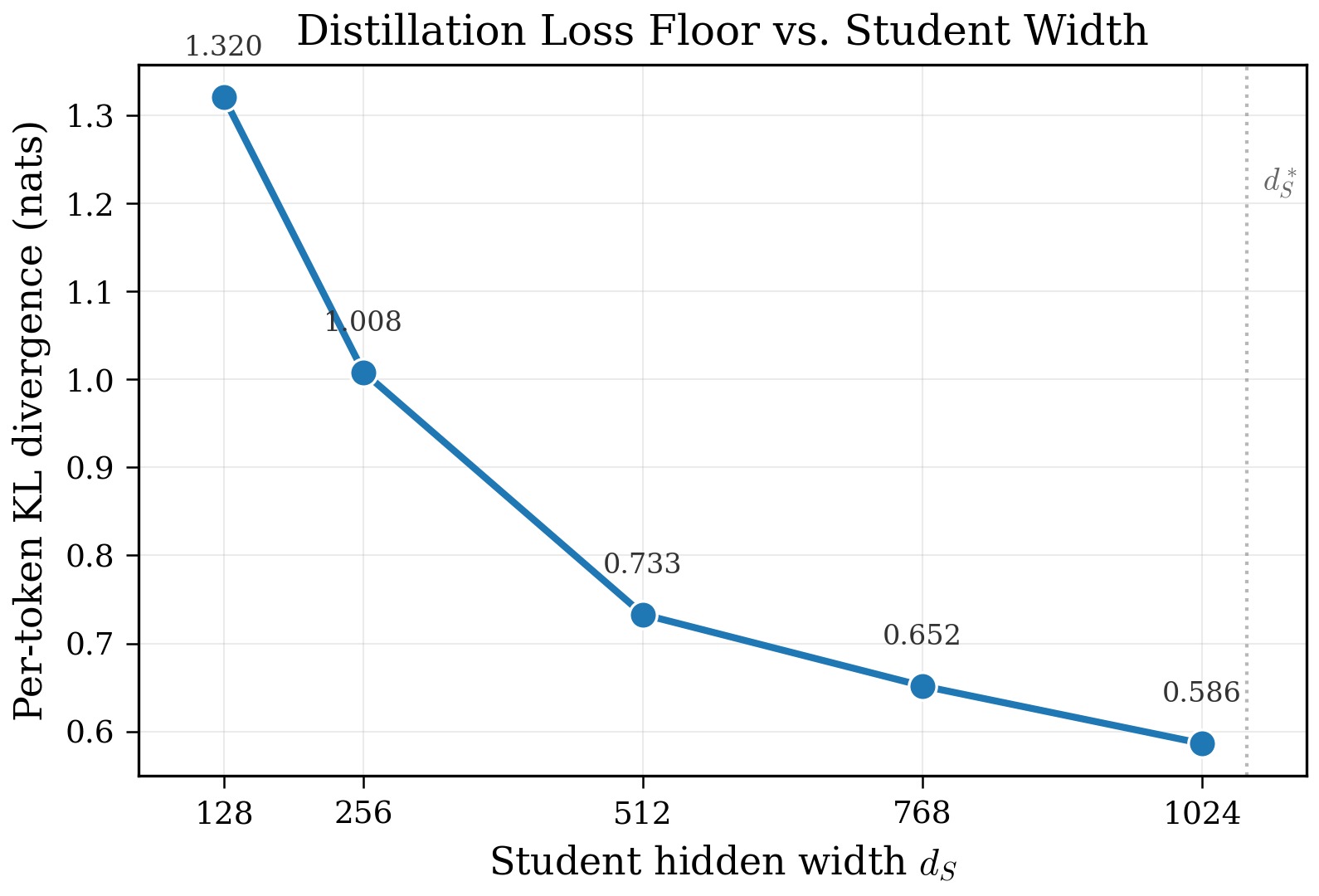}
    \caption{Per-token KL floor vs.\ width}
\end{subfigure}
\end{center}
\vspace{-1.5em}
\caption{Distillation results. \textbf{(a)} Eval loss for all widths; narrower students plateau higher. \textbf{(b)} Floor decreases from 1.320 to 0.586 nats. Dotted line marks $\dstar \approx 1065$.}
\label{fig:distill_results}
\end{figure}

\subsection{Two-component floor decomposition}

The formula predicts floors in hidden-space importance units; distillation loss is in KL nats. Figure~\ref{fig:calibration} reveals the key insight: the observed floor decomposes into two independent components:
\begin{equation}
    L_\mathrm{observed} = C \cdot \hat{L}^*_\mathrm{geometric} + B
    \label{eq:decomposition}
\end{equation}
The baseline $B$ is not a fitted parameter: we estimate it directly from the $\dS = 1024$ control (same width as teacher), which gives $B = 0.586$ nats/token as an independent measurement requiring only one distillation run. With $B$ fixed, the model reduces to a single free parameter $C$, fit to the remaining four points. Table~\ref{tab:calibration} summarizes: the affine fit achieves $R^2 = 0.993$ with $C = 8.97$ and fitted $B = 0.623$, within 6\% of the independently measured $B = 0.586$, confirming consistency. A pure linear model (no baseline) fails catastrophically ($R^2 = -1.982$). Practically, if loss is near $B$, width increases will not help; if well above $B$, Eq.~\ref{eq:floor} predicts how much wider students help.

\begin{table}[t]
\begin{center}
\setlength{\tabcolsep}{8pt}
\renewcommand{\arraystretch}{1.1}
\begin{tabular}{lcc}
\toprule
\textbf{Fit} & \textbf{Parameters} & \textbf{$R^2$} \\
\midrule
Linear ($y = Cx$)     & $C = 8.97$              & $-1.982$ \\
Affine ($y = Cx + B$) & $C = 8.97$,\ $B = 0.623$ & $0.993$ \\
\bottomrule
\end{tabular}
\end{center}
\vspace{-0.5em}
\caption{Calibration fits. The affine model succeeds because the floor has two components: geometric ($C \cdot \hat{L}^*$, width-dependent) and architectural baseline ($B$, width-independent).}
\label{tab:calibration}
\end{table}

\begin{figure}[t]
\begin{center}
\includegraphics[width=0.88\linewidth]{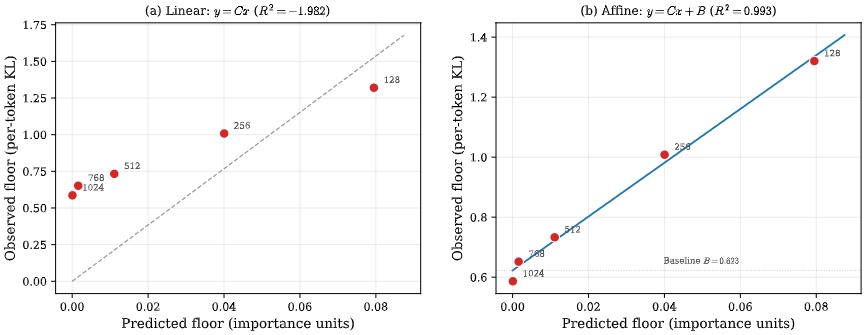}
\end{center}
\vspace{-1.5em}
\caption{Two-component decomposition. \textbf{(a)} Linear fit ($R^2 = -1.982$) fails. \textbf{(b)} Affine fit ($R^2 = 0.993$): $\text{observed} = 8.97 \times \text{predicted} + 0.623$. Baseline $B = 0.623$: architectural floor; slope $C = 8.97$: amplification through transformer layers.}
\label{fig:calibration}
\end{figure}

Table~\ref{tab:decomp} makes this quantitative: geometry accounts for 56\% at $\dS = 128$ but $<$1\% at $\dS = 1024$. The observed floor also follows a power-law scaling:
\begin{equation}
    L_\mathrm{obs}(\dS) = 11.6 \cdot \dS^{-0.47} + 0.13 \quad (R^2 = 0.998)
    \label{eq:scaling}
\end{equation}
where $\gamma = 0.47$ reflects the importance distribution's power-law tail ($\beta \approx 3.05$). Each doubling of width reduces the geometric component by ${\sim}28\%$. The normalized predicted curve (Figure~\ref{fig:pred_vs_obs}) drops faster than observed because it tracks only the geometric component; observed floors converge to $B/L_\mathrm{obs}(128) \approx 0.44$.

\begin{table}[t]
\begin{center}
\setlength{\tabcolsep}{4pt}
\renewcommand{\arraystretch}{1.05}
\begin{tabular}{lcccc}
\toprule
$\dS$ & \textbf{Observed} & $C\hat{L}^*$ & $B$ & \textbf{\% Geom.} \\
\midrule
128  & 1.320 & 0.713 & 0.586 & 55.6\% \\
256  & 1.008 & 0.359 & 0.586 & 41.9\% \\
512  & 0.733 & 0.100 & 0.586 & 20.1\% \\
768  & 0.652 & 0.014 & 0.586 & 10.1\% \\
1024 & 0.586 & 0.001 & 0.586 & 0.1\% \\
\bottomrule
\end{tabular}
\end{center}
\vspace{-0.5em}
\caption{Floor decomposition into geometric ($C\hat{L}^*$) and baseline ($B$) components.}
\label{tab:decomp}
\end{table}

\section{Experiment 4: linear probing}
\label{sec:exp4}

Experiments 2--3 show \emph{what} happens (floors appear where predicted) and \emph{how much} (the two-component decomposition). Linear probing tests whether the floor arises from geometric feature absence.

\subsection{Method}

We select six binary concepts spanning varying prevalence: \textit{is this a question?}, \textit{is this French?}, \textit{contains code?}, \textit{about sports?}, \textit{legal text?}, and \textit{medical text?}. For each, we collect 2{,}000 positive and 2{,}000 negative examples from The Pile. We extract layer-12 hidden states, average-pool across the sequence, and train logistic regression probes (80/20 split) on the teacher and students at widths 128, 768, and 1024.

\subsection{Results}

Table~\ref{tab:probing} reports probe accuracies. All six concepts remain linearly decodable even at $\dS = 128$ ($81\times$ compression), with mean absolute change of only 1.27 pp. No concept drops to chance. The lower teacher accuracy for \textit{is this a question?} (74.0\%) reflects genuine ambiguity: interrogative syntax relies on subtle token-level patterns rather than global semantic content.
\begin{table}[H]
\begin{center}
\setlength{\tabcolsep}{8pt}
\renewcommand{\arraystretch}{1.1}
\begin{tabular}{lcccc}
\toprule
\textbf{Concept} & \textbf{Teacher} & $\dS{=}1024$ & $\dS{=}768$ & $\dS{=}128$ \\
\midrule
Is question?    & 74.0 & 77.1 & 75.9 & 75.0 \\
Is French?      & 99.4 & 99.2 & 99.5 & 98.9 \\
Contains code?  & 84.0 & 83.8 & 83.6 & 86.9 \\
About sports?   & 96.5 & 97.2 & 97.2 & 94.1 \\
Legal text?     & 97.8 & 97.2 & 97.8 & 97.4 \\
Medical text?   & 97.6 & 97.4 & 97.5 & 97.1 \\
\bottomrule
\end{tabular}
\end{center}
\vspace{-0.5em}
\caption{Linear probe accuracy (\%) at layer 12. All concepts stay far above chance (50\%).}
\label{tab:probing}
\end{table}

Figure~\ref{fig:probe_results} visualizes these results. The accuracy-change plot reveals subtle trade-offs: \textit{about sports?} drops 2.4 pp at $\dS = 128$ while \textit{contains code?} \emph{increases} by 2.9 pp, suggesting capacity reallocation under pressure.
\begin{figure}[H]
\centering
\begin{subfigure}[b]{0.30\linewidth}
    \includegraphics[width=\linewidth]{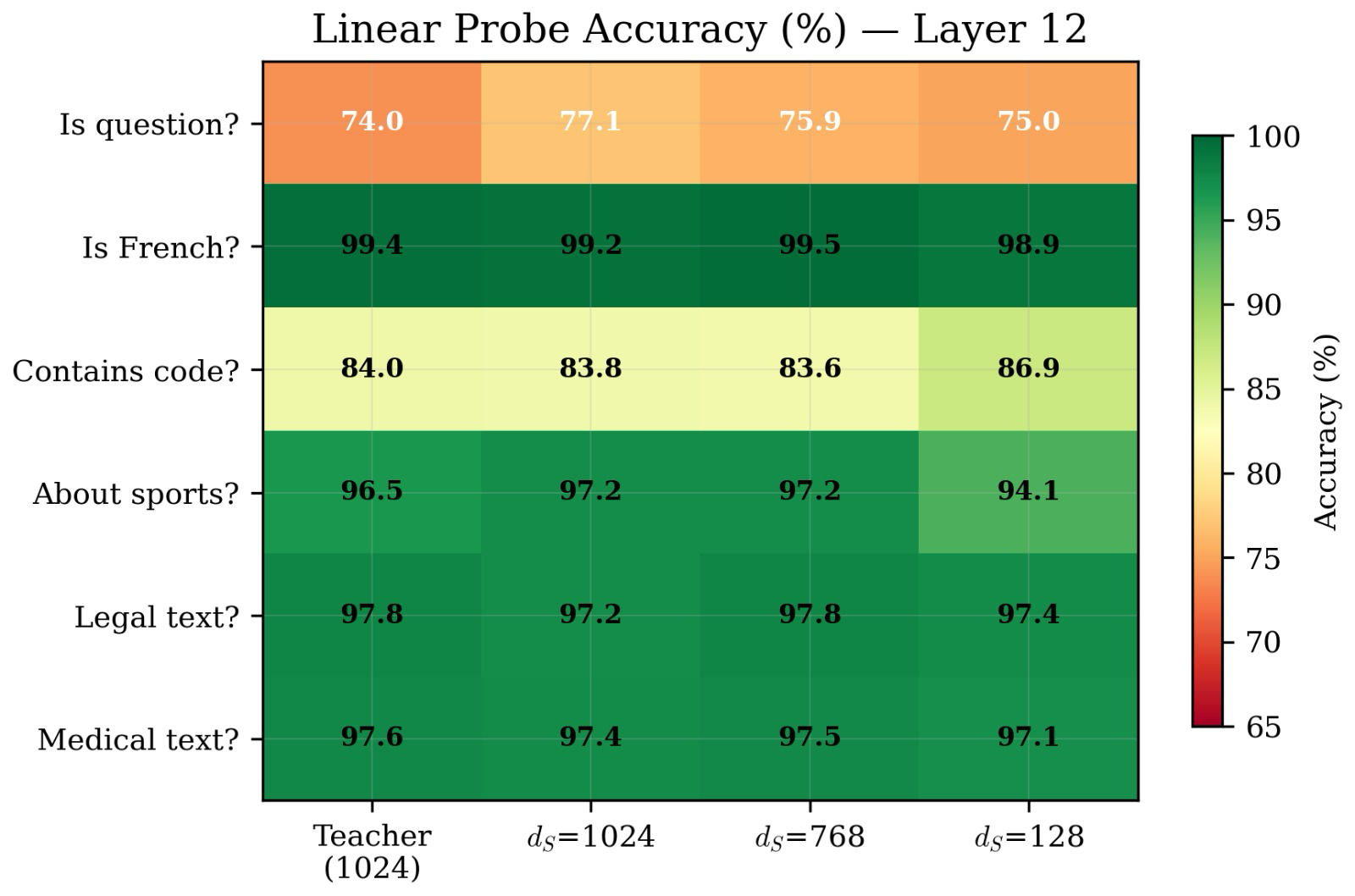}
    \caption{Accuracy heatmap}
\end{subfigure}
\hfill
\begin{subfigure}[b]{0.30\linewidth}
    \includegraphics[width=\linewidth]{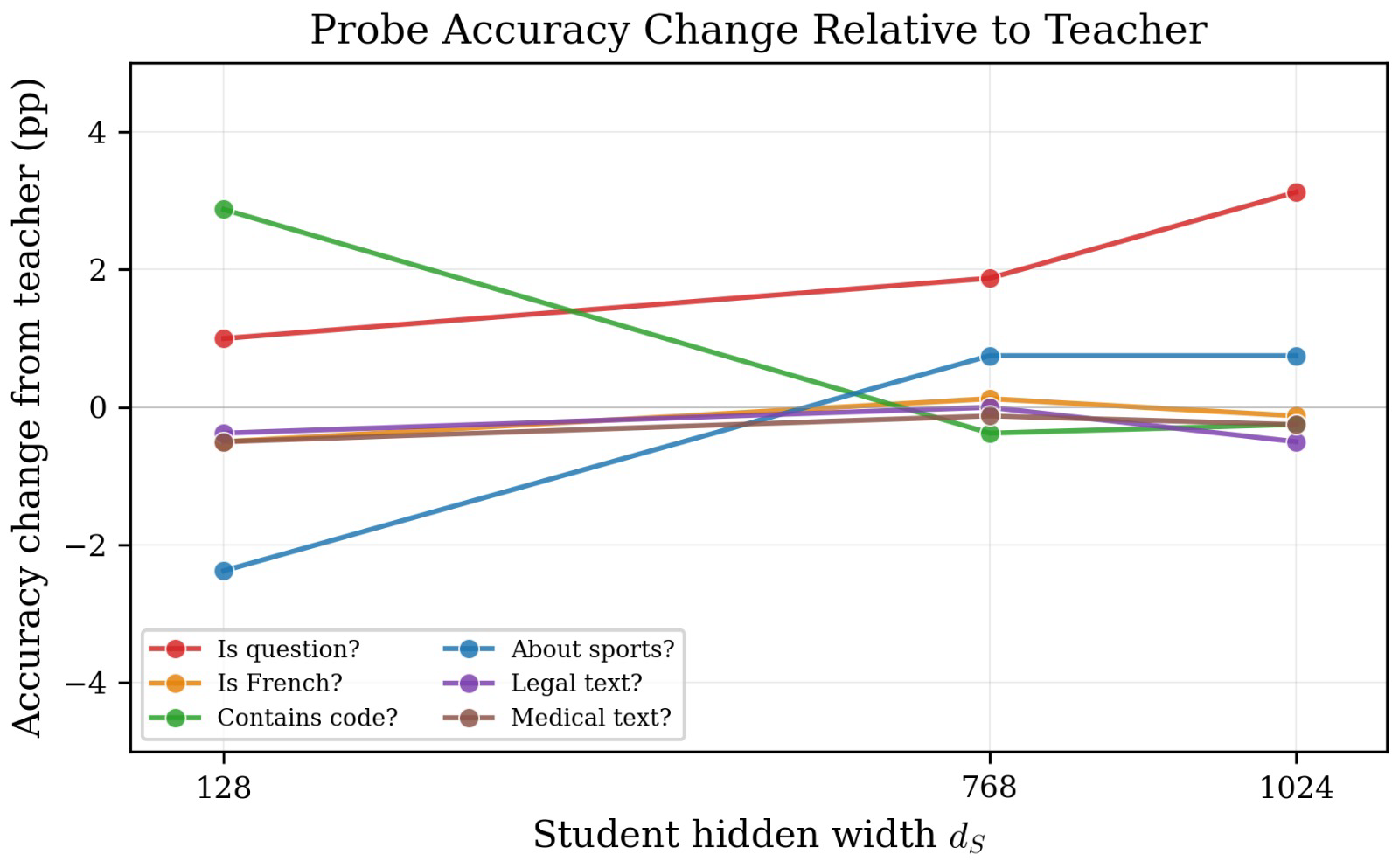}
    \caption{Change vs.\ teacher}
\end{subfigure}
\hfill
\begin{subfigure}[b]{0.30\linewidth}
    \includegraphics[width=\linewidth]{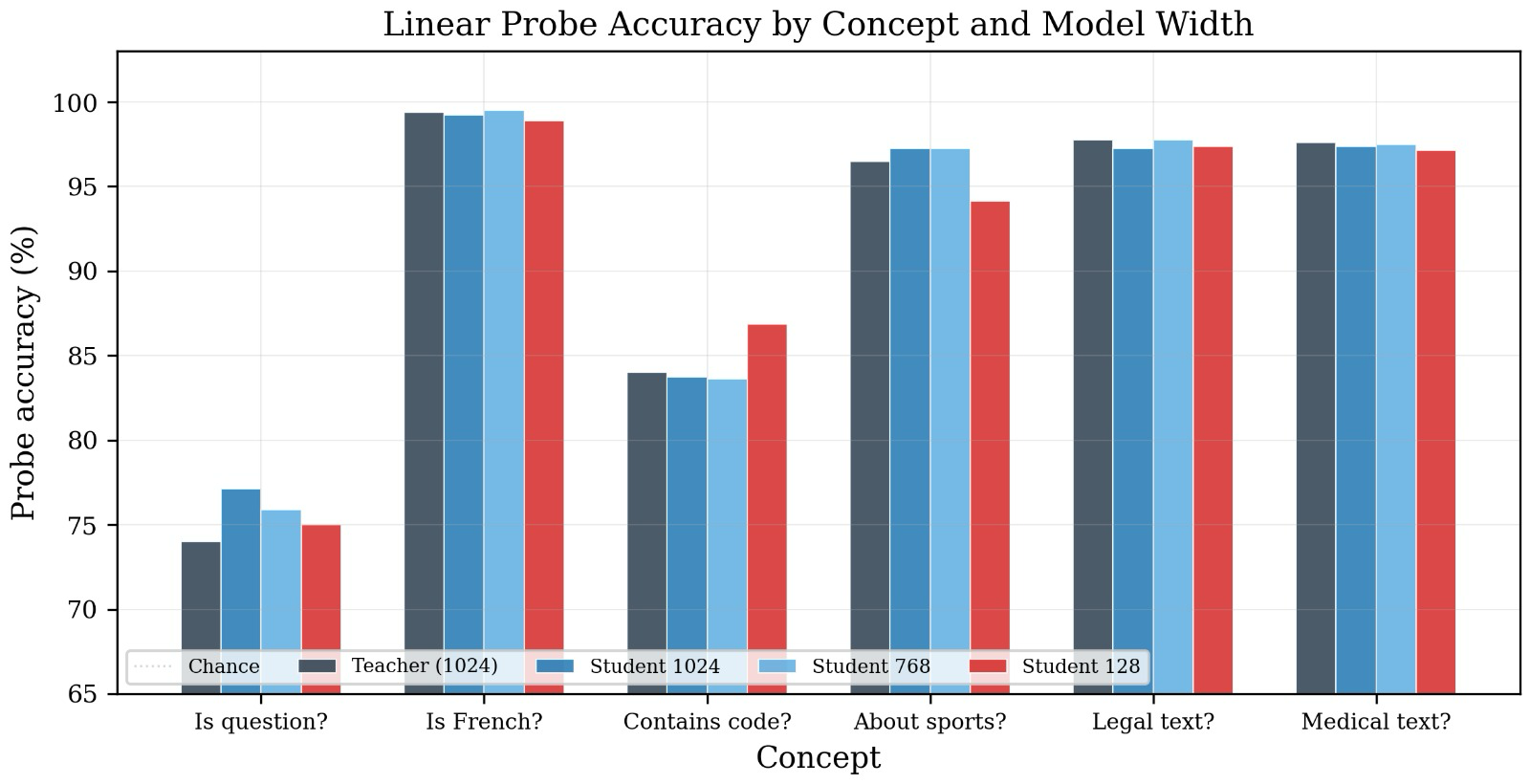}
    \caption{Absolute accuracy}
\end{subfigure}
\vspace{-0.5em}
\caption{Linear probe results. \textbf{(a)} Heatmap: all concepts survive compression. \textbf{(b)} $\pm 3$ pp shifts reflect reallocation. \textbf{(c)} All above chance (50\%).}
\label{fig:probe_results}
\end{figure}

\subsection{Interpretation: the granularity mismatch}

The key insight is a \emph{granularity mismatch} between the concepts we probe and the features the bottleneck drops. Each coarse domain (e.g., ``French text'') is supported by hundreds of SAE features. At $\dS = 128$, 3{,}446 features survive and 25{,}219 are dropped, but enough high-importance features persist within each domain. \emph{The floor is not caused by losing any single recognizable capability}, but from aggregate loss of thousands of fine-grained features, each contributing negligibly but summing to measurable KL.

\section{Discussion and conclusion}
\label{sec:discussion}

For students below $\dstar$, \emph{no training method can help}: the bottleneck is dimensional \citep{busbridge2025distillation}. Table~\ref{tab:decomp} and Eq.~\ref{eq:scaling} give practitioners a complete toolkit: measure SAE statistics, predict the floor at any width, and determine whether the target loss is achievable. $B$ cannot be reduced by width, only by changing the objective \citep{romero2015fitnets} or depth. The amplification $C = 8.97$ is consistent with a naive estimate $\sqrt{12} \cdot \ln(50304/1024) \approx 13.5$, suggesting $C$ is dominated by vocabulary expansion. Probing individual SAE features for the predicted staircase dropout is a key future direction.

\paragraph{Limitations.}
(1) Hard-cutoff approximation: ${\sim}5$--$15\%$ error near the phase transition.
(2) Multi-layer extension empirically validated, not proven.
(3) Width-only compression at fixed depth; single SAE expansion (32$\times$).
(4) Coarse probes only; feature-level verification needed.
(5) Theorem~\ref{thm:main} assumes the hidden layer approximates a random projection, which holds approximately.

\section*{Acknowledgments}

The authors used Claude Opus (Anthropic) for figure generation, \LaTeX{} formatting, and proofreading. All content was reviewed and verified by the authors.

\bibliographystyle{plainnat}
\bibliography{references}

\appendix

\section{Sparsity capacity function}
\label{app:galpha}

\begin{table}[H]
\begin{center}
\setlength{\tabcolsep}{8pt}
\renewcommand{\arraystretch}{1.1}
\begin{tabular}{cccc}
\toprule
$\alpha$ & $1-\alpha$ & $g(\alpha)$ & Features/dim \\
\midrule
0.00  & 1.000  & 1.0   & 1.0$\times$ \\
0.50  & 0.500  & 2.9   & 2.9$\times$ \\
0.80  & 0.200  & 3.1   & 3.1$\times$ \\
0.90  & 0.100  & 4.3   & 4.3$\times$ \\
0.95  & 0.050  & 6.7   & 6.7$\times$ \\
0.99  & 0.010  & 21.7  & 21.7$\times$ \\
0.992 & 0.008  & 27.0  & 27.0$\times$ \\
0.999 & 0.001  & 145   & 145$\times$ \\
\bottomrule
\end{tabular}
\end{center}
\caption{Reference values for $g(\alpha)$.}
\label{tab:galpha}
\end{table}

\section{Additional toy model results}
\label{app:toy}

\paragraph{Sparsity effect.} Higher sparsity yields lower floors at every width because $g(\alpha)$ packs more features per dimension (Figure~\ref{fig:sparsity_effect}).

\begin{figure}[H]
\begin{center}
\includegraphics[width=\linewidth]{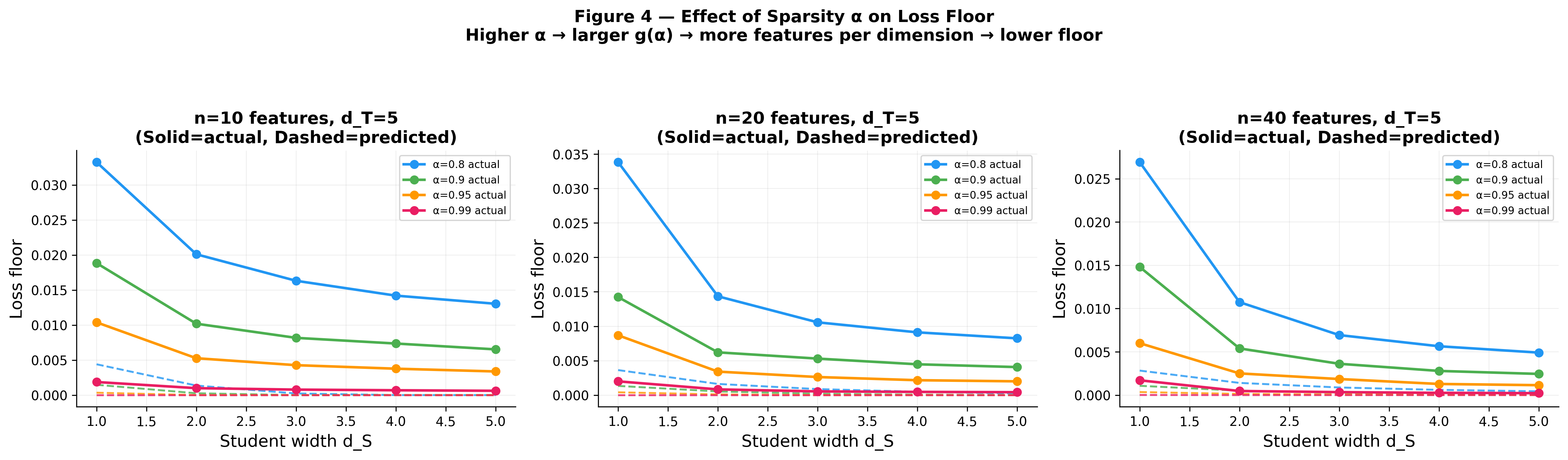}
\end{center}
\vspace{-0.5em}
\caption{Effect of $\alpha$ on floor at $\dT = 5$ for $n \in \{10, 20, 40\}$. Higher sparsity (more features/dim) yields lower floors. Solid = actual; dashed = predicted.}
\label{fig:sparsity_effect}
\end{figure}
\paragraph{Error distributions.} The refined formula concentrates errors near 100\% accuracy across all sparsities; the naive formula degrades at high $\alpha$ (Figure~\ref{fig:error_dist}).

\begin{figure}[H]
\begin{center}
\includegraphics[width=\linewidth]{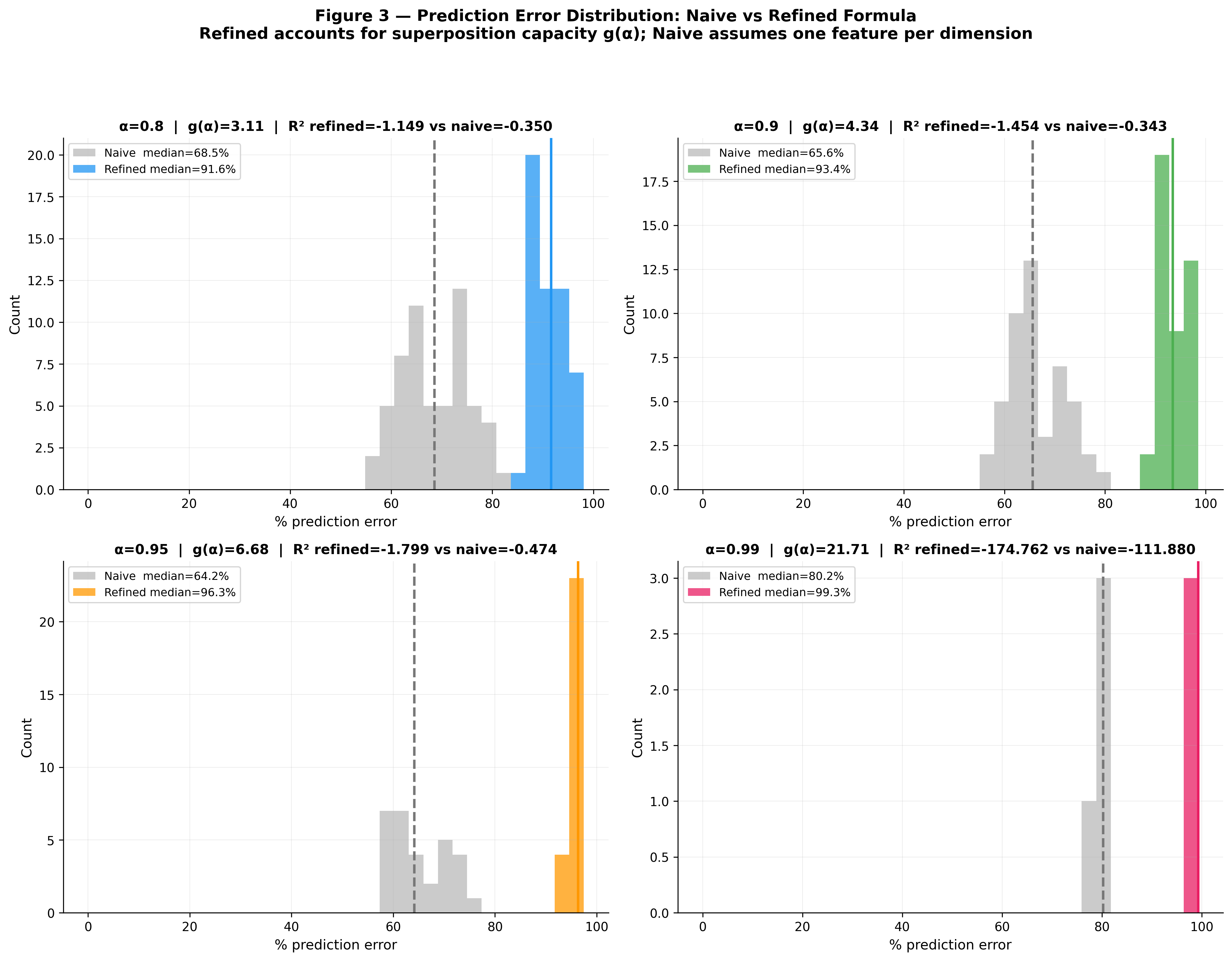}
\end{center}
\vspace{-0.5em}
\caption{Prediction error by sparsity. Refined (colored): median $>90\%$ accuracy; naive (gray): degrades at high $\alpha$.}
\label{fig:error_dist}
\end{figure}
\paragraph{Error heatmap.} The refined formula achieves $>93\%$ accuracy in nearly all configurations, reaching 100\% at $\alpha = 0.99$ (Figure~\ref{fig:error_heatmap}).

\begin{figure}[H]
\begin{center}
\includegraphics[width=\linewidth]{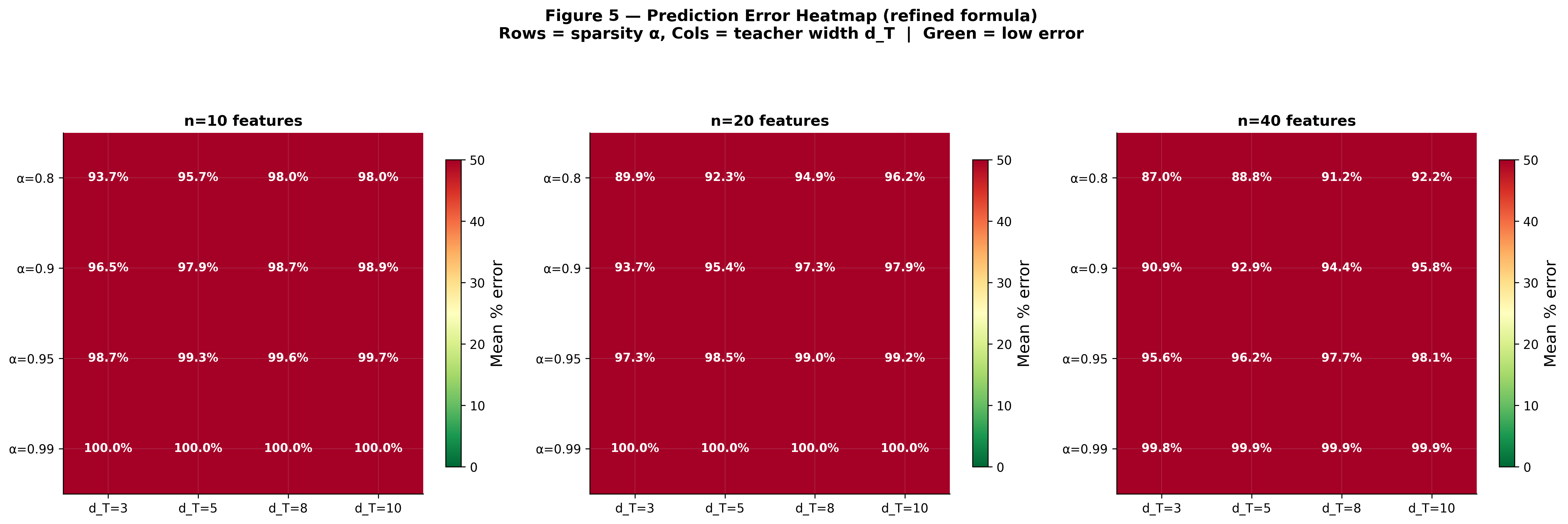}
\end{center}
\vspace{-0.5em}
\caption{Accuracy heatmap (refined). Rows: $\alpha$; cols: $\dT$; panels: $n$. Green = $>99\%$.}
\label{fig:error_heatmap}
\end{figure}
\paragraph{Zipf importance.} The toy model uses $I_i \propto 1/i$, matching real SAE distributions (Figure~\ref{fig:zipf}).

\begin{figure}[H]
\begin{center}
\includegraphics[width=\linewidth]{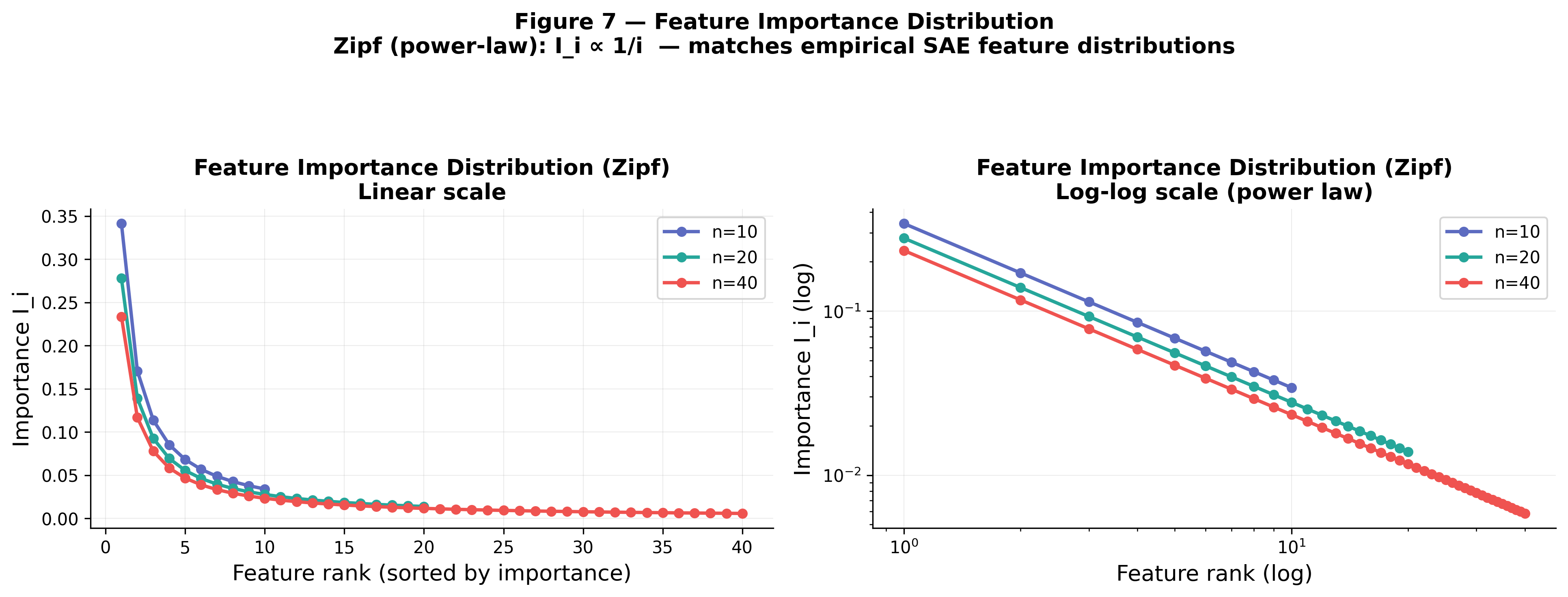}
\end{center}
\vspace{-0.5em}
\caption{Zipf importance. \textbf{Left:} Linear scale. \textbf{Right:} Log-log confirms power law.}
\label{fig:zipf}
\end{figure}
\paragraph{Universal scaling.} When plotted against $\dS/\dstar$, all configurations collapse onto one curve (Figure~\ref{fig:toy_critical}): floor ${\sim}1$ for $\dS \ll \dstar$, sharp drop at $\dS = \dstar$, zero beyond.

\begin{figure}[H]
\begin{center}
\includegraphics[width=0.65\linewidth]{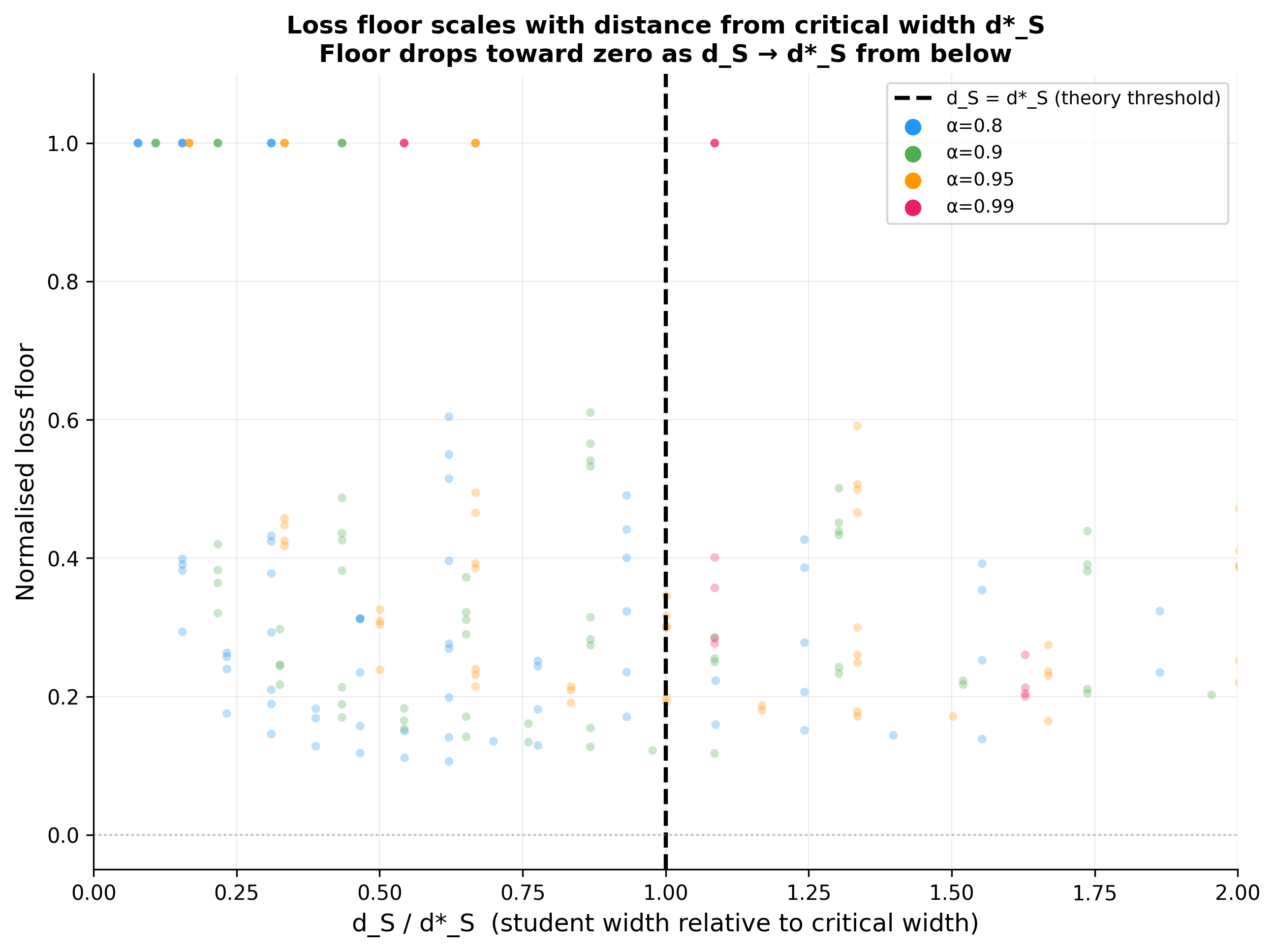}
\end{center}
\vspace{-0.5em}
\caption{Normalized floor vs.\ $\dS/\dstar$. All configurations collapse: floor drops sharply at $\dS = \dstar$ (dashed). This universal scaling confirms the phase transition.}
\label{fig:toy_critical}
\end{figure}
\paragraph{Training dynamics.} Students converge to distinct floors within ${\sim}200$ steps, confirming the floor is capacity-limited, not training-limited (Figure~\ref{fig:training_curves}).

\begin{figure}[H]
\begin{center}
\includegraphics[width=\linewidth]{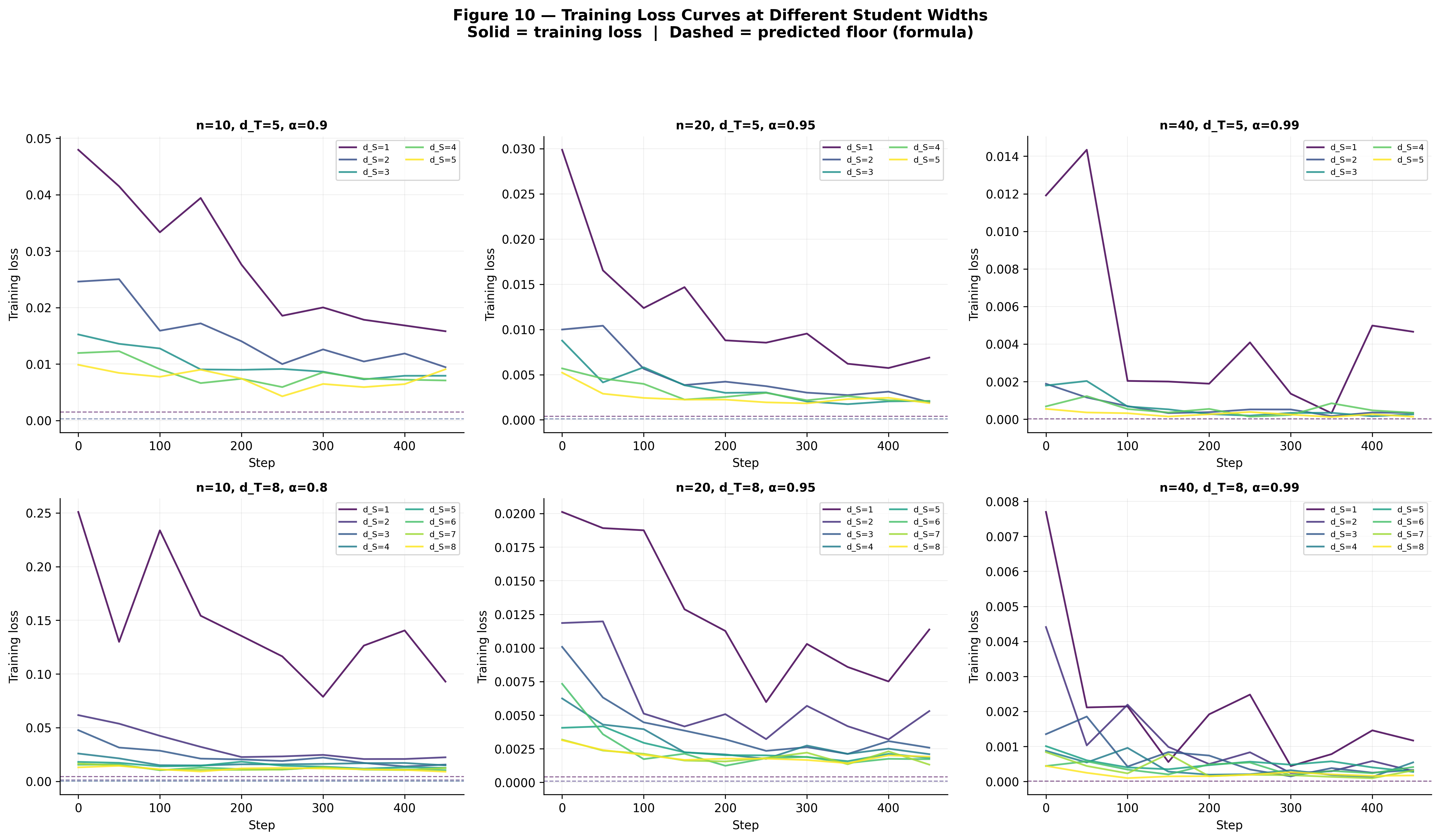}
\end{center}
\vspace{-0.5em}
\caption{Training curves at different widths for six configurations. Dashed = predicted floors. Rapid convergence confirms geometric origin.}
\label{fig:training_curves}
\end{figure}

\section{Student architecture details}
\label{app:students}
\vspace{-1em}
\begin{table}[H]
\begin{center}
\setlength{\tabcolsep}{8pt}
\renewcommand{\arraystretch}{1.1}
\begin{tabular}{lccccc}
\toprule
$\dS$ & Layers & Heads & FFN dim & Params & Ratio \\
\midrule
128  & 24 & 8  & 512   & ${\sim}18$M  & 23$\times$ \\
256  & 24 & 16 & 1024  & ${\sim}45$M  & 9$\times$ \\
512  & 24 & 16 & 2048  & ${\sim}127$M & 3.2$\times$ \\
768  & 24 & 16 & 3072  & ${\sim}247$M & 1.6$\times$ \\
1024 & 24 & 16 & 4096  & ${\sim}405$M & 1.0$\times$ \\
\bottomrule
\end{tabular}
\end{center}
\caption{Student architectures. All share teacher's depth (24), vocabulary (50{,}304), and positional encoding.}
\label{tab:students}
\end{table}

\section{SAE training details}
\label{app:sae}
\vspace{-1em}
Architecture: pre-encoder bias $b_\mathrm{pre} \in \R^{1024}$, encoder $W_\mathrm{enc} \in \R^{32768 \times 1024}$, decoder $W_\mathrm{dec} \in \R^{1024 \times 32768}$ (ReLU, unit-norm decoder columns). Training: Adam ($\eta = 3 \times 10^{-4}$, $\beta_1 = 0.9$, $\beta_2 = 0.999$), gradient clipping at 1.0, $\lambda = 8 \times 10^{-4}$ (summed over features, averaged over batch), 300M tokens, batches of $32 \times 1024$.

\begin{figure}[H]
\begin{center}
\begin{subfigure}[b]{0.24\linewidth}
    \includegraphics[width=\linewidth]{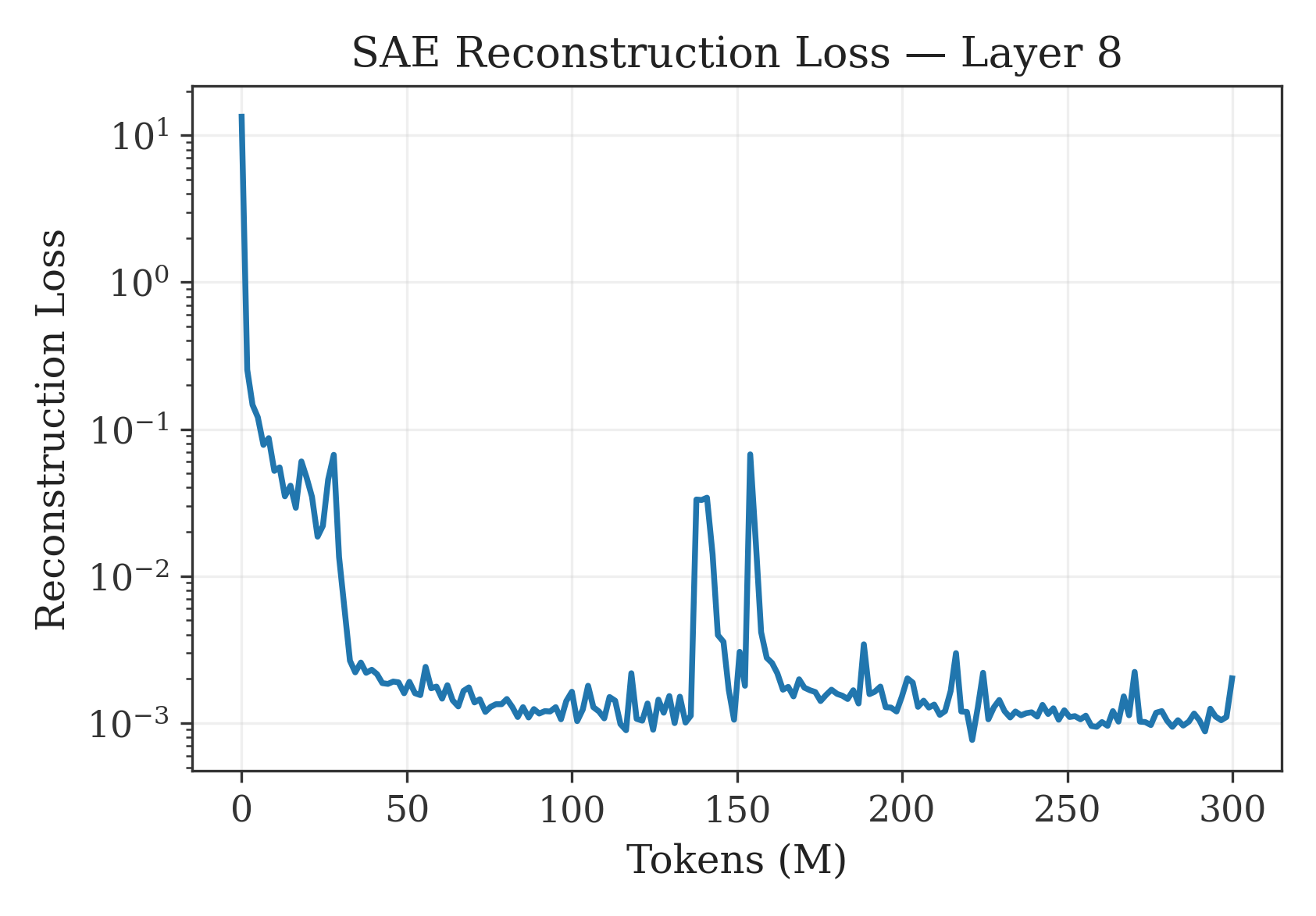}
    \caption{Recon, L8}
\end{subfigure}
\hfill
\begin{subfigure}[b]{0.24\linewidth}
    \includegraphics[width=\linewidth]{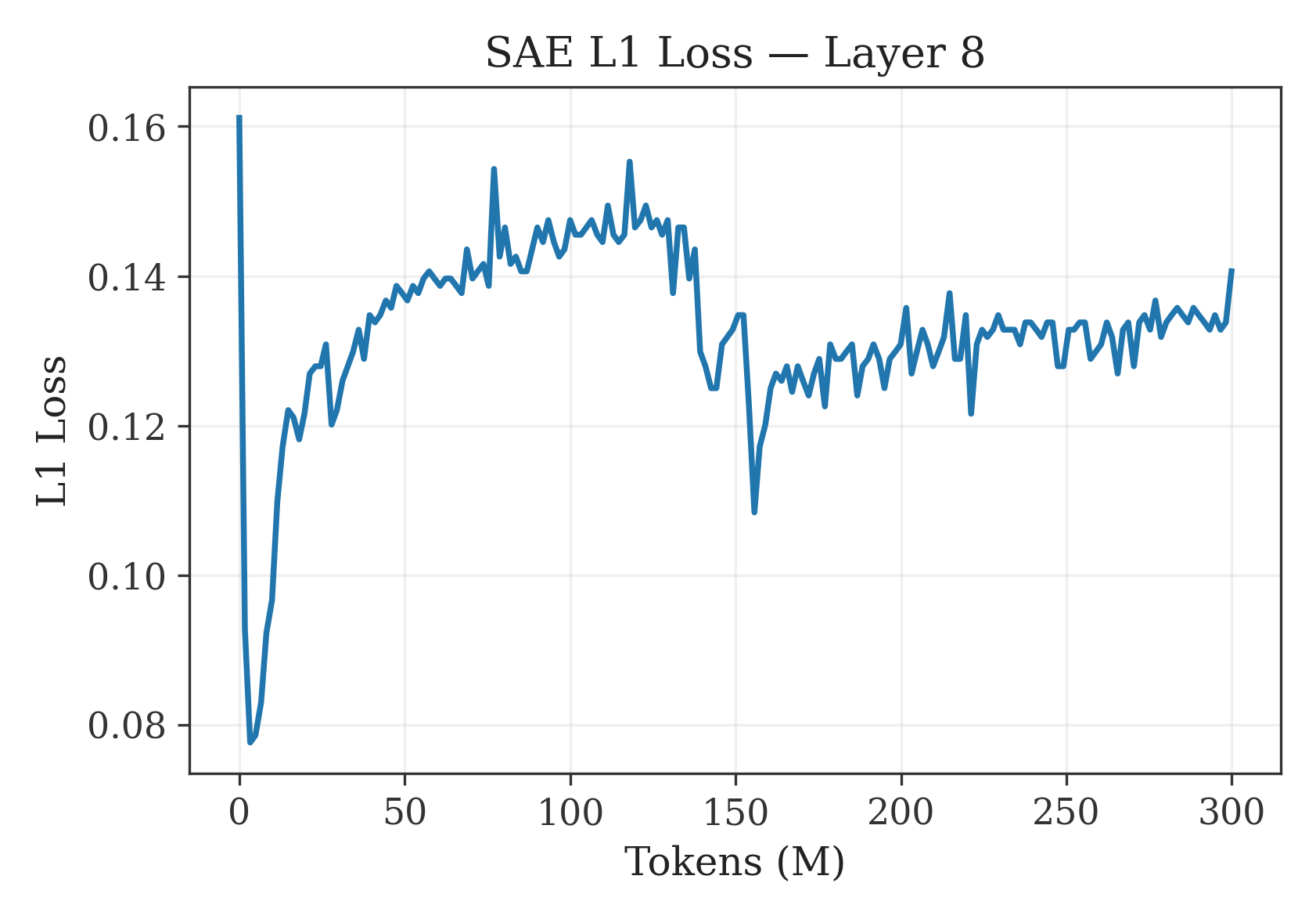}
    \caption{L1, L8}
\end{subfigure}
\hfill
\begin{subfigure}[b]{0.24\linewidth}
    \includegraphics[width=\linewidth]{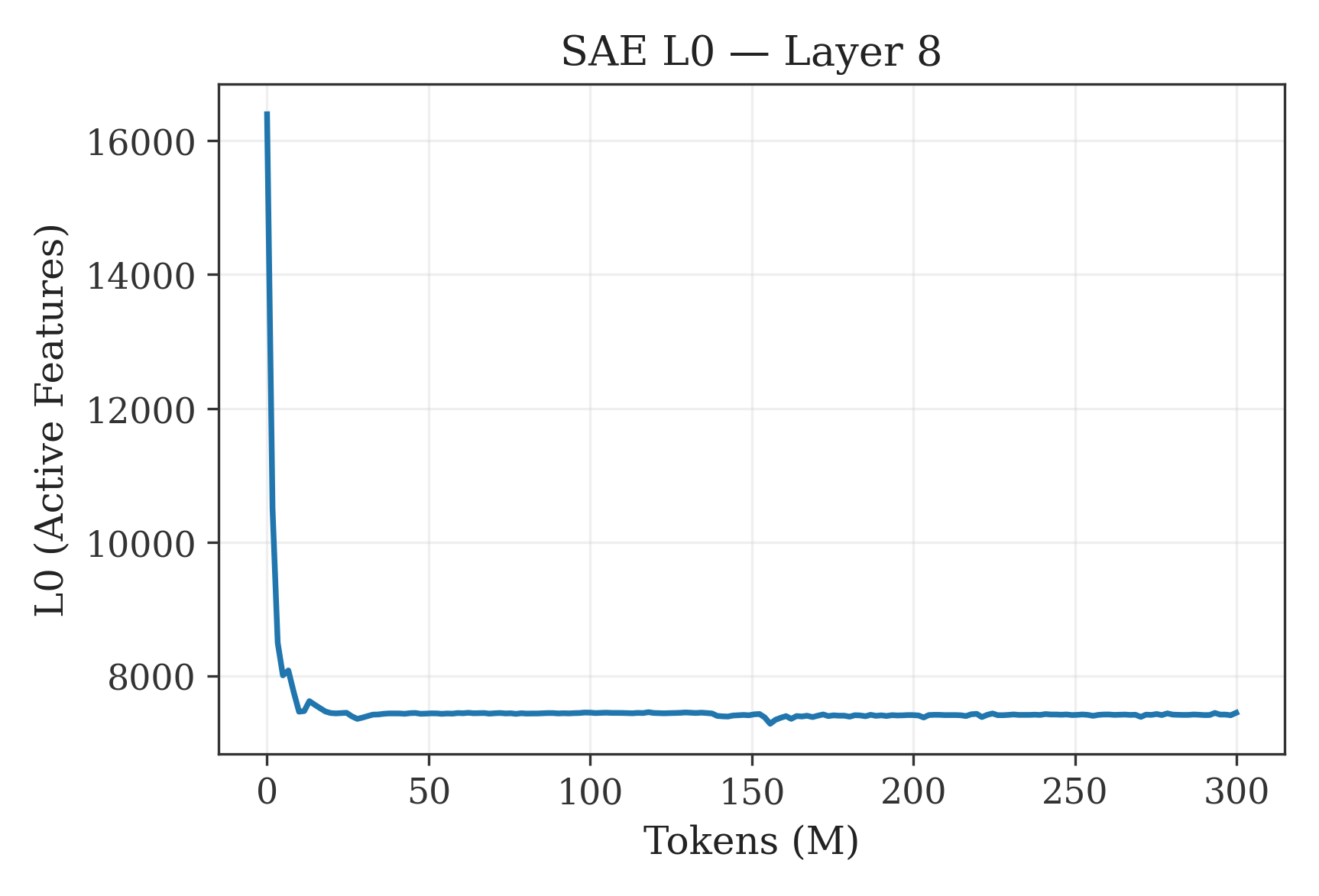}
    \caption{$L_0$, L8}
\end{subfigure}
\hfill
\begin{subfigure}[b]{0.24\linewidth}
    \includegraphics[width=\linewidth]{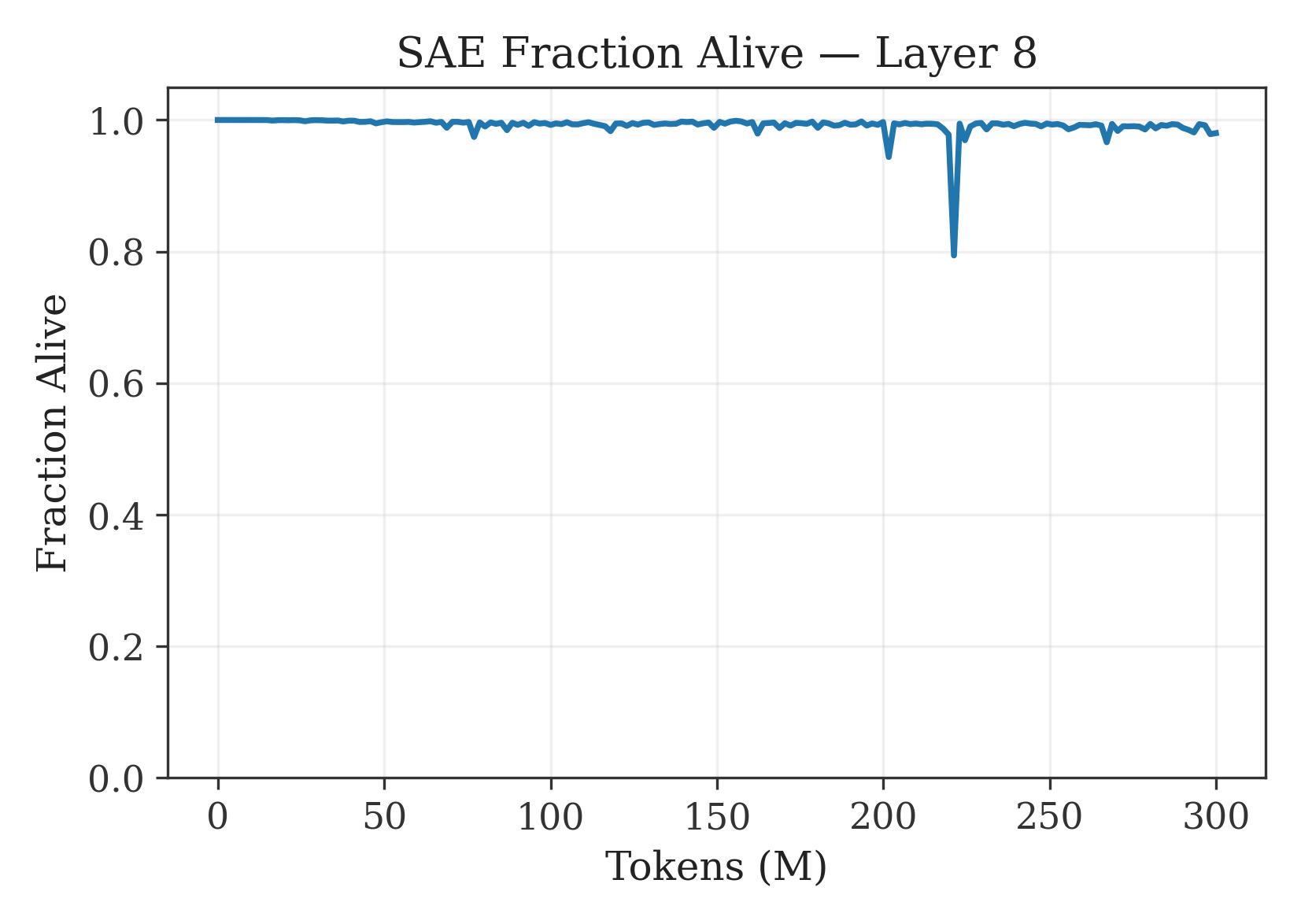}
    \caption{Alive, L8}
\end{subfigure}
\vspace{0.3em}
\begin{subfigure}[b]{0.24\linewidth}
    \includegraphics[width=\linewidth]{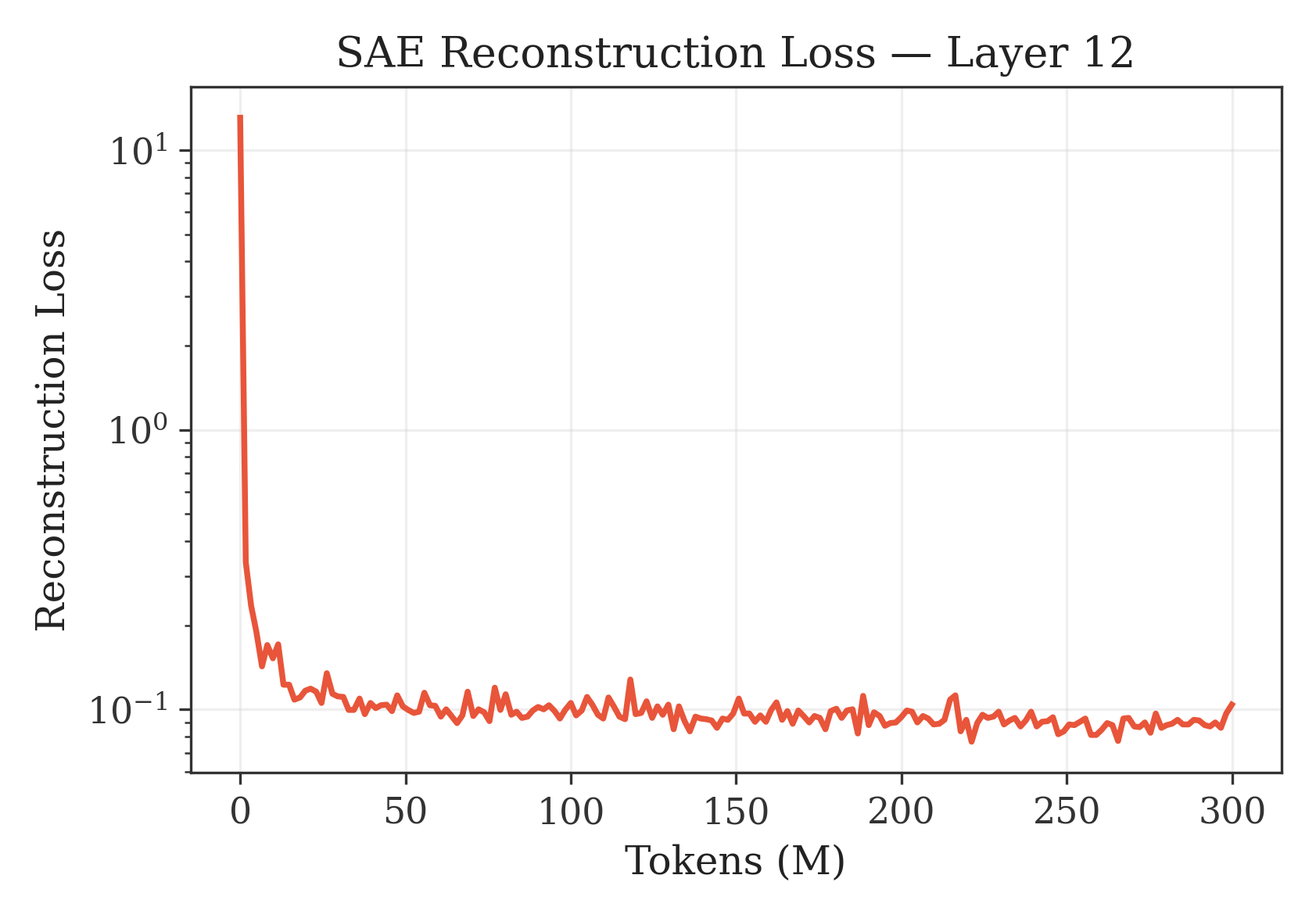}
    \caption{Recon, L12}
\end{subfigure}
\hfill
\begin{subfigure}[b]{0.24\linewidth}
    \includegraphics[width=\linewidth]{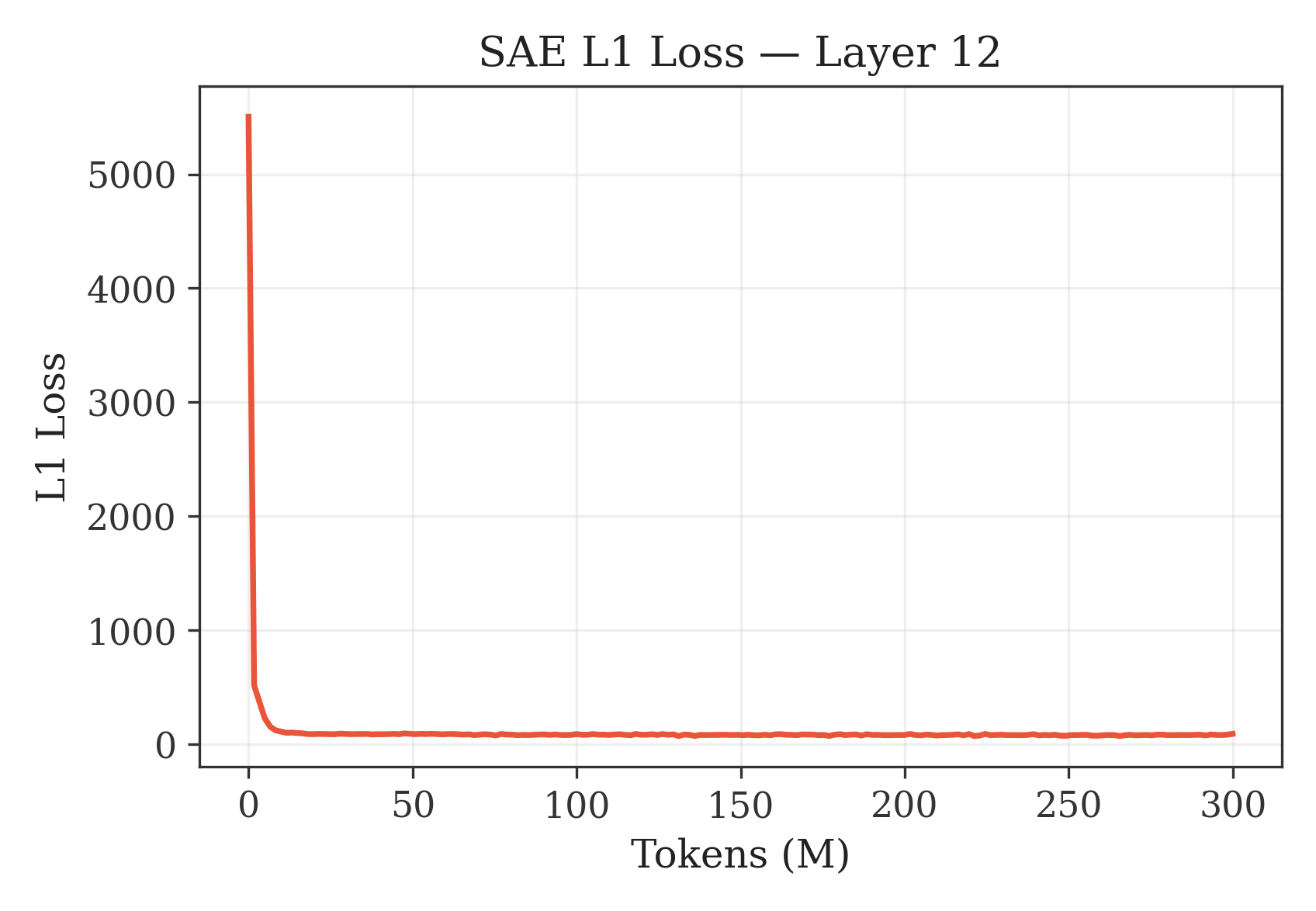}
    \caption{L1, L12}
\end{subfigure}
\hfill
\begin{subfigure}[b]{0.24\linewidth}
    \includegraphics[width=\linewidth]{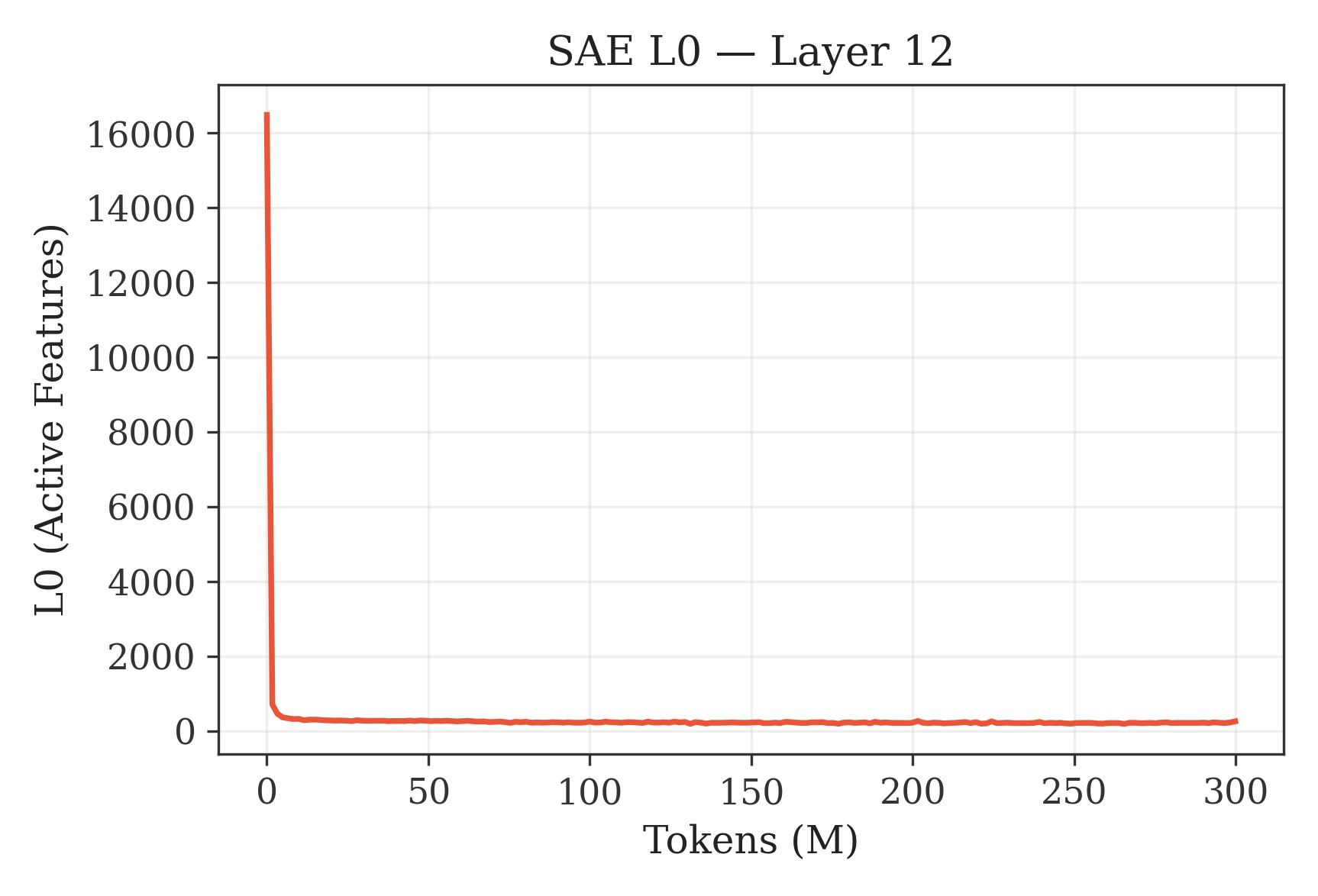}
    \caption{$L_0$, L12}
\end{subfigure}
\hfill
\begin{subfigure}[b]{0.24\linewidth}
    \includegraphics[width=\linewidth]{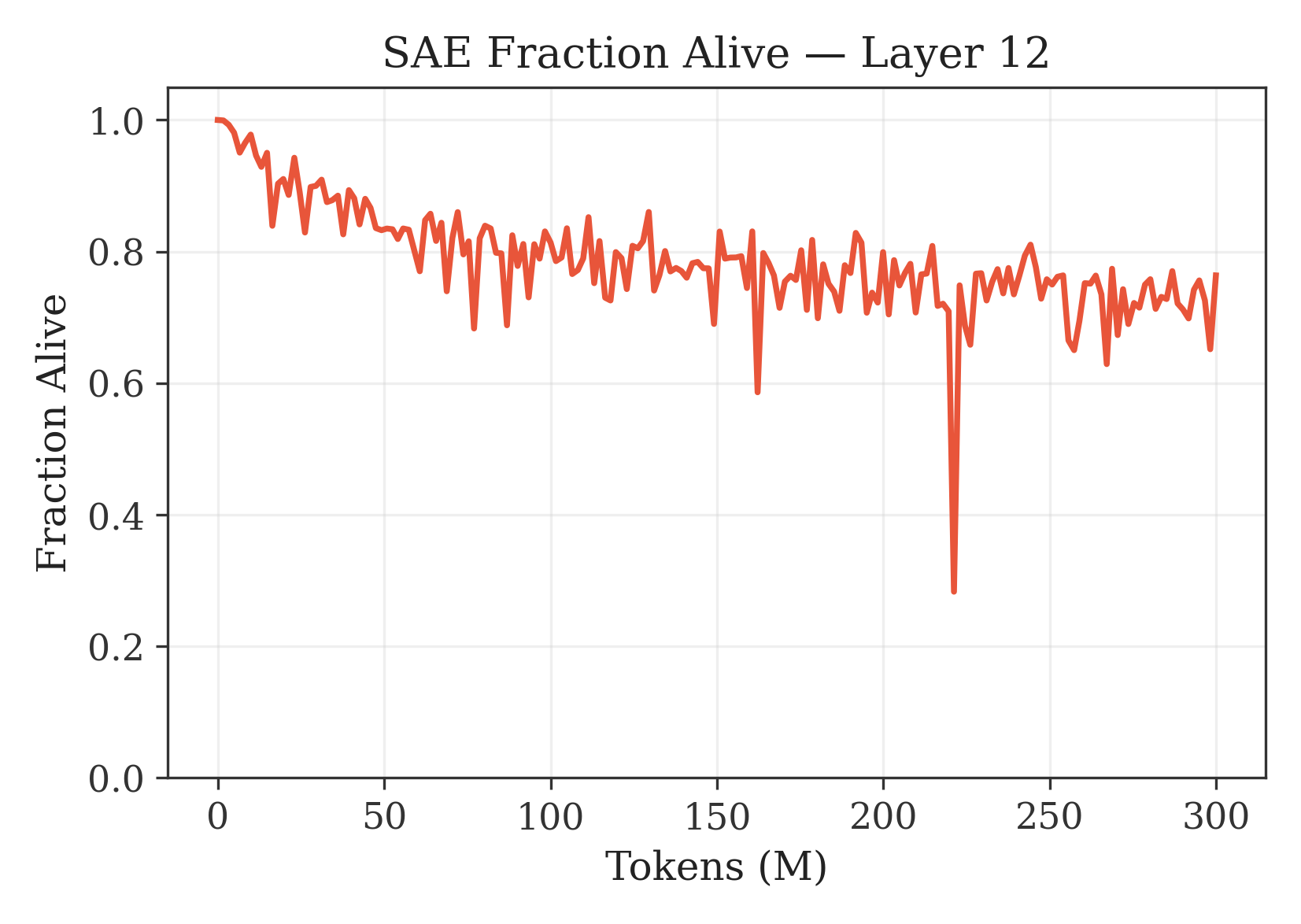}
    \caption{Alive, L12}
\end{subfigure}
\vspace{0.3em}
\begin{subfigure}[b]{0.24\linewidth}
    \includegraphics[width=\linewidth]{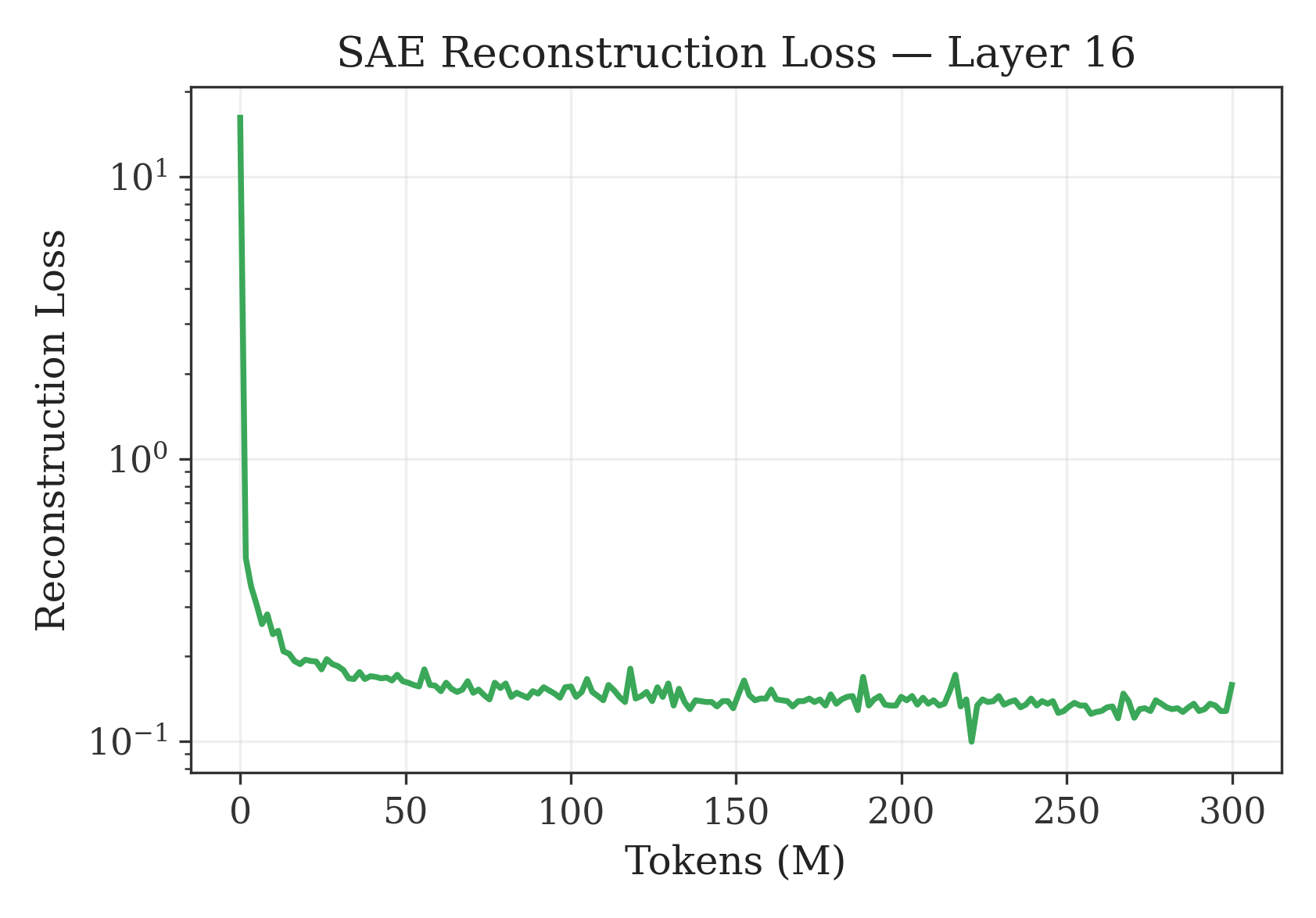}
    \caption{Recon, L16}
\end{subfigure}
\hfill
\begin{subfigure}[b]{0.24\linewidth}
    \includegraphics[width=\linewidth]{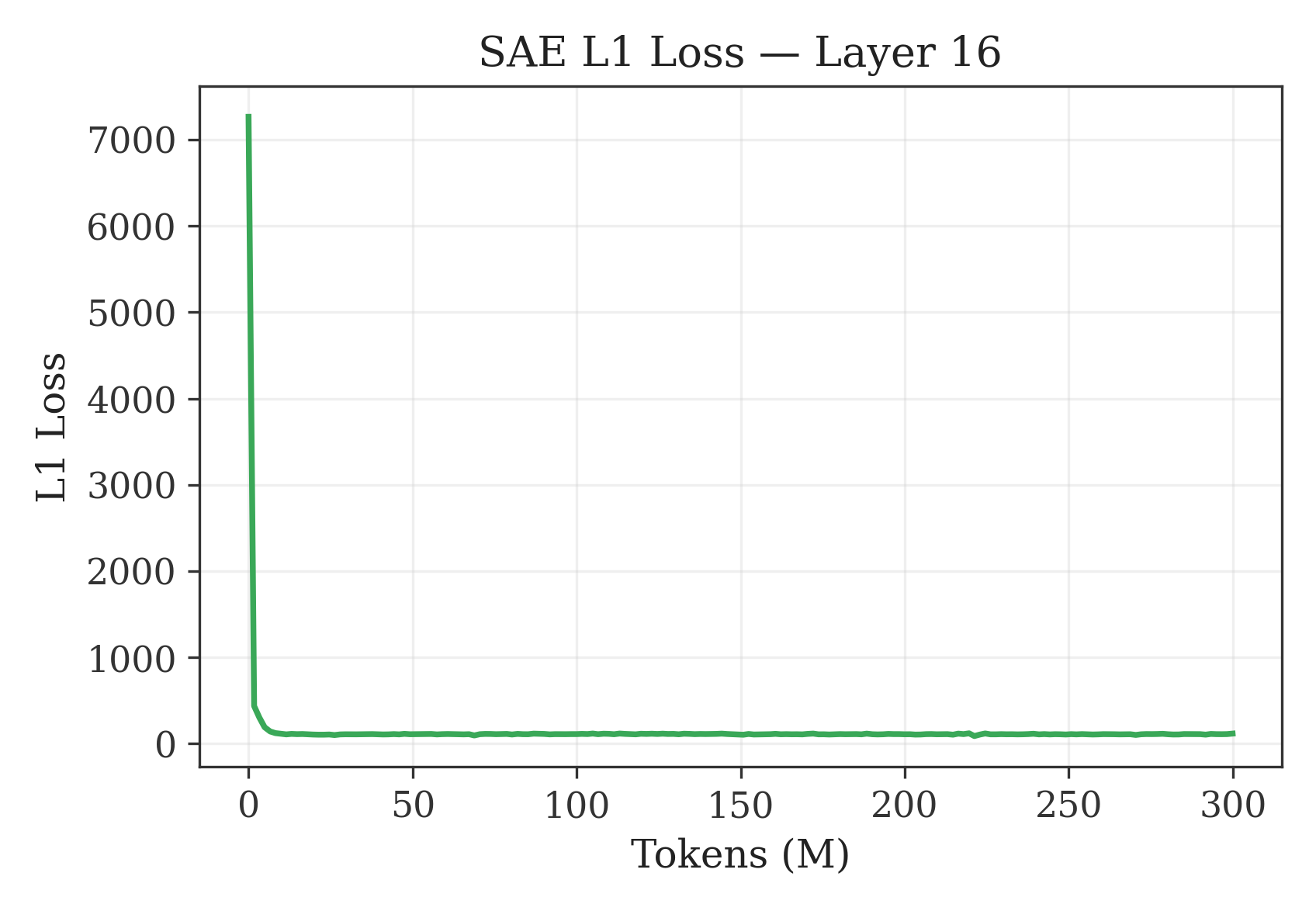}
    \caption{L1, L16}
\end{subfigure}
\hfill
\begin{subfigure}[b]{0.24\linewidth}
    \includegraphics[width=\linewidth]{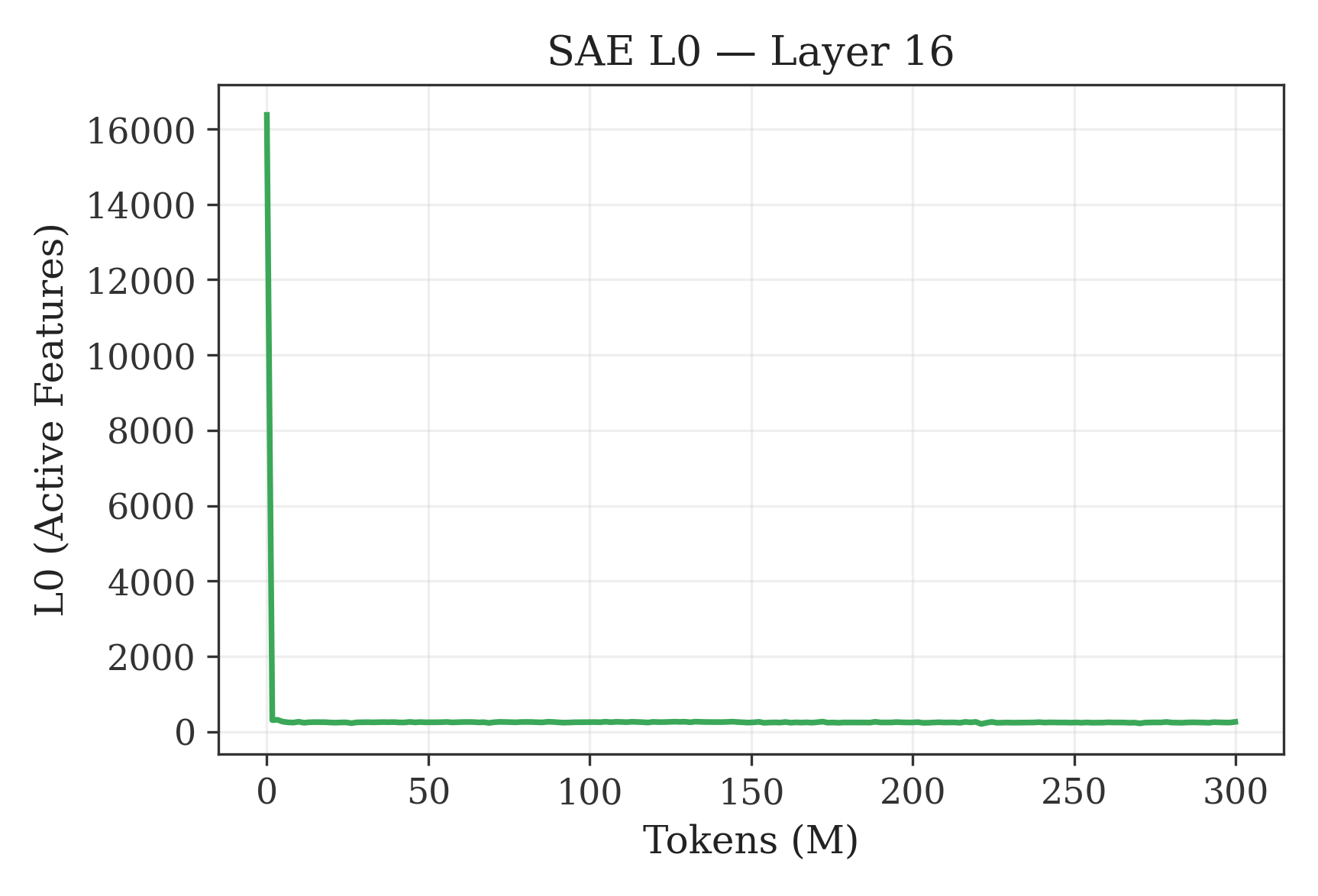}
    \caption{$L_0$, L16}
\end{subfigure}
\hfill
\begin{subfigure}[b]{0.24\linewidth}
    \includegraphics[width=\linewidth]{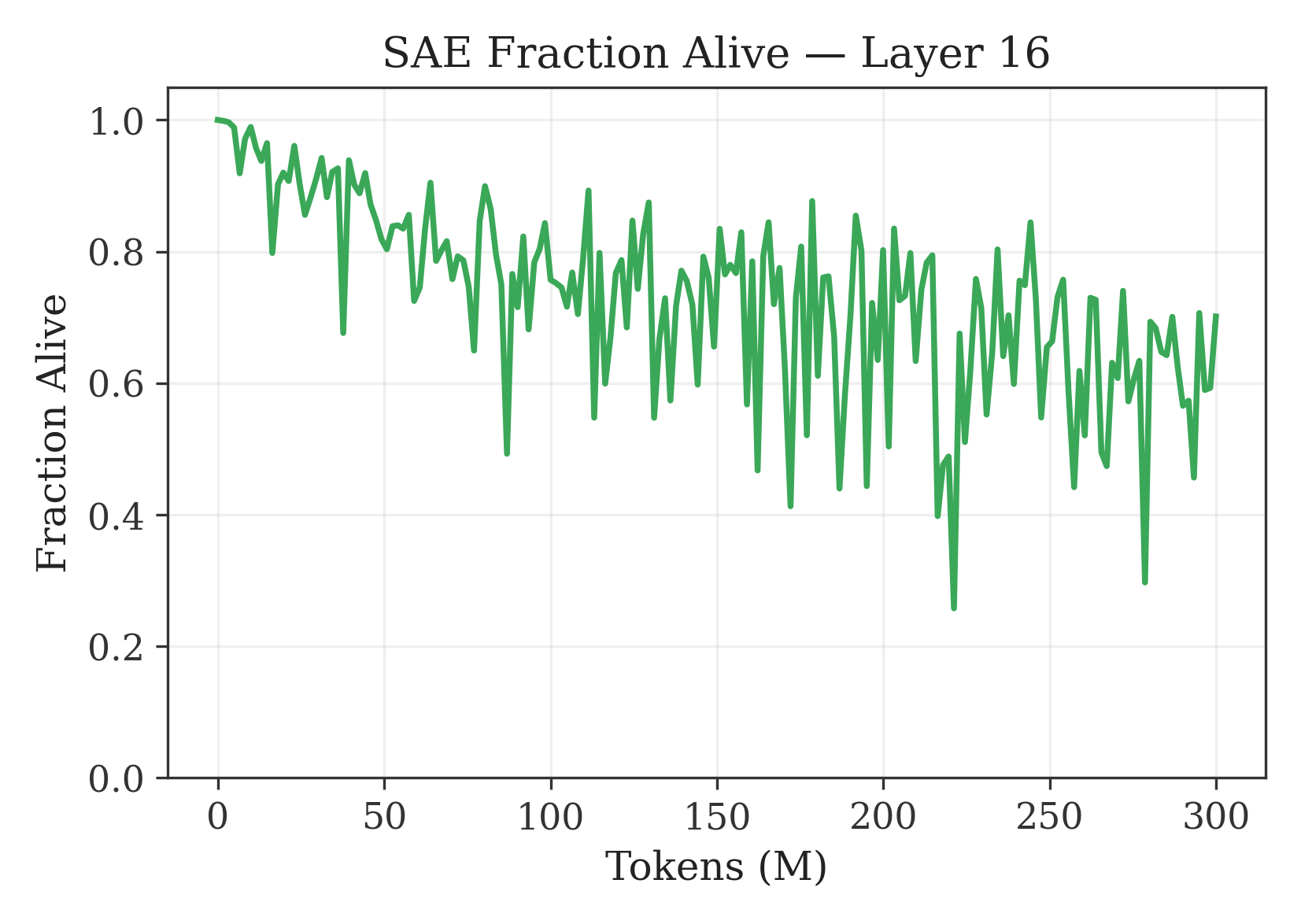}
    \caption{Alive, L16}
\end{subfigure}
\end{center}
\caption{Per-layer SAE curves: layers 8 (top), 12 (mid), 16 (bottom). Layer 8 has lower recon loss, higher $L_0$, near-zero feature death.}
\label{fig:sae_training_individual}
\end{figure}

\section{Distillation training details}
\label{app:distill}

KL distillation at $T=2$, AdamW ($\eta = 3\times10^{-4}$, decay 0.01), warmup 1{,}000 steps, cosine decay, 30{,}000 steps, batch $32 \times 512$. Floor = mean eval loss over final 10\%.

\begin{figure}[H]
\begin{center}
\begin{subfigure}[b]{0.48\linewidth}
    \includegraphics[width=\linewidth]{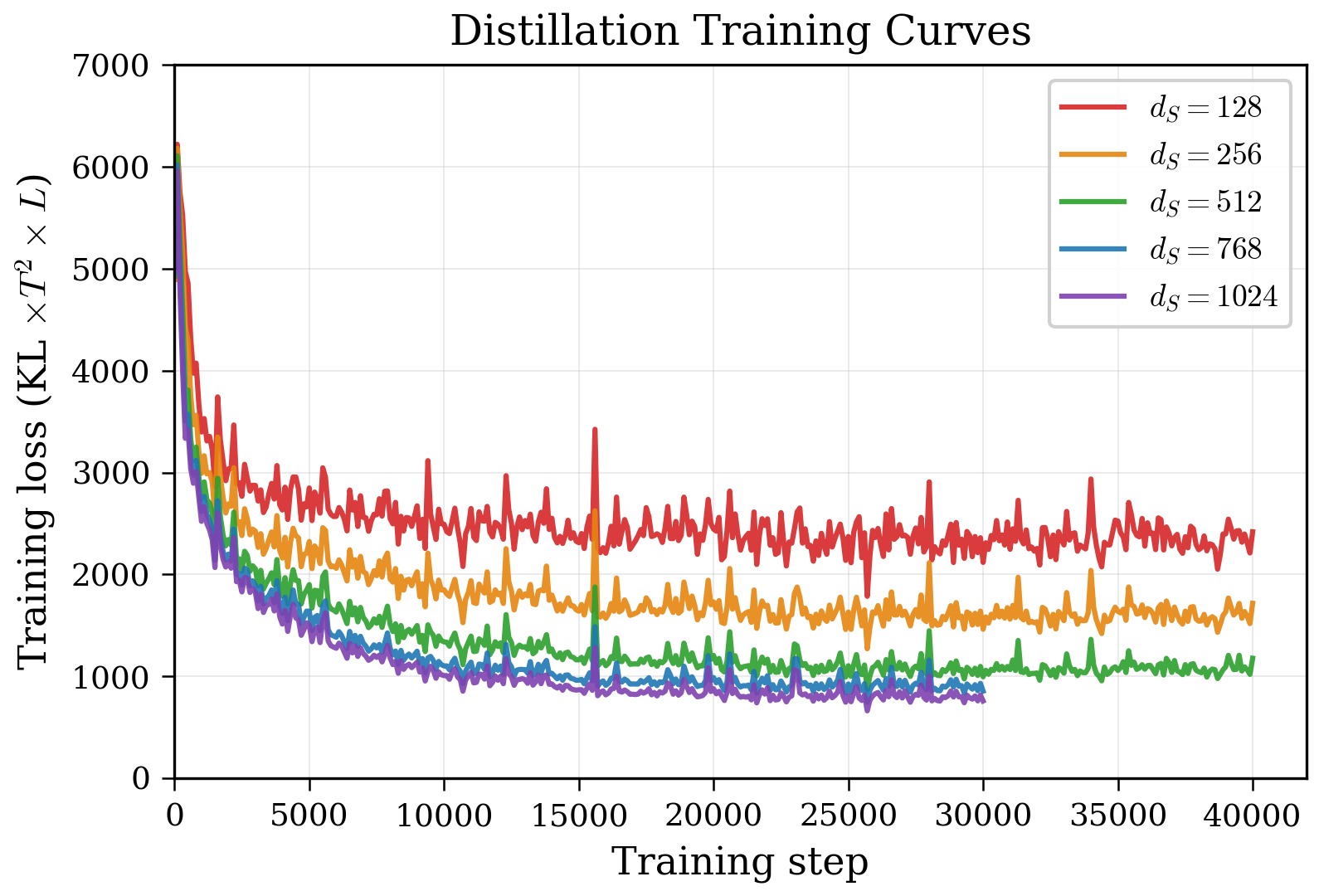}
    \caption{Training loss curves}
\end{subfigure}
\hfill
\begin{subfigure}[b]{0.48\linewidth}
    \includegraphics[width=\linewidth]{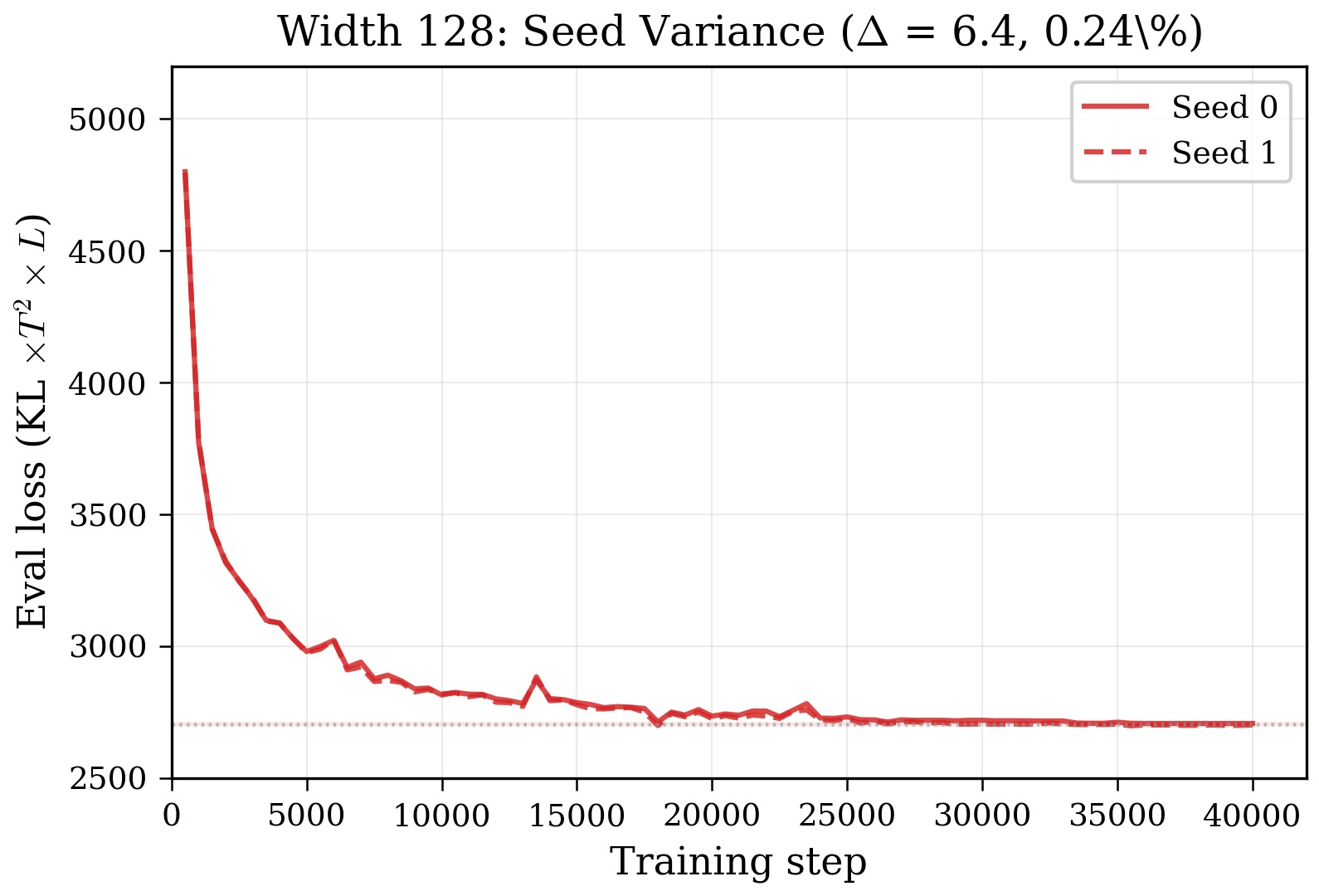}
    \caption{Seed variance ($\dS = 128$)}
\end{subfigure}
\end{center}
\vspace{-0.5em}
\caption{\textbf{(a)} Training loss for all widths. \textbf{(b)} Two seeds at $\dS=128$: floors differ by $\Delta = 6.4$ (0.24\%), confirming the floor is deterministic.}
\label{fig:distill_details}
\end{figure}

\begin{figure}[H]
\begin{center}
\includegraphics[width=0.65\linewidth]{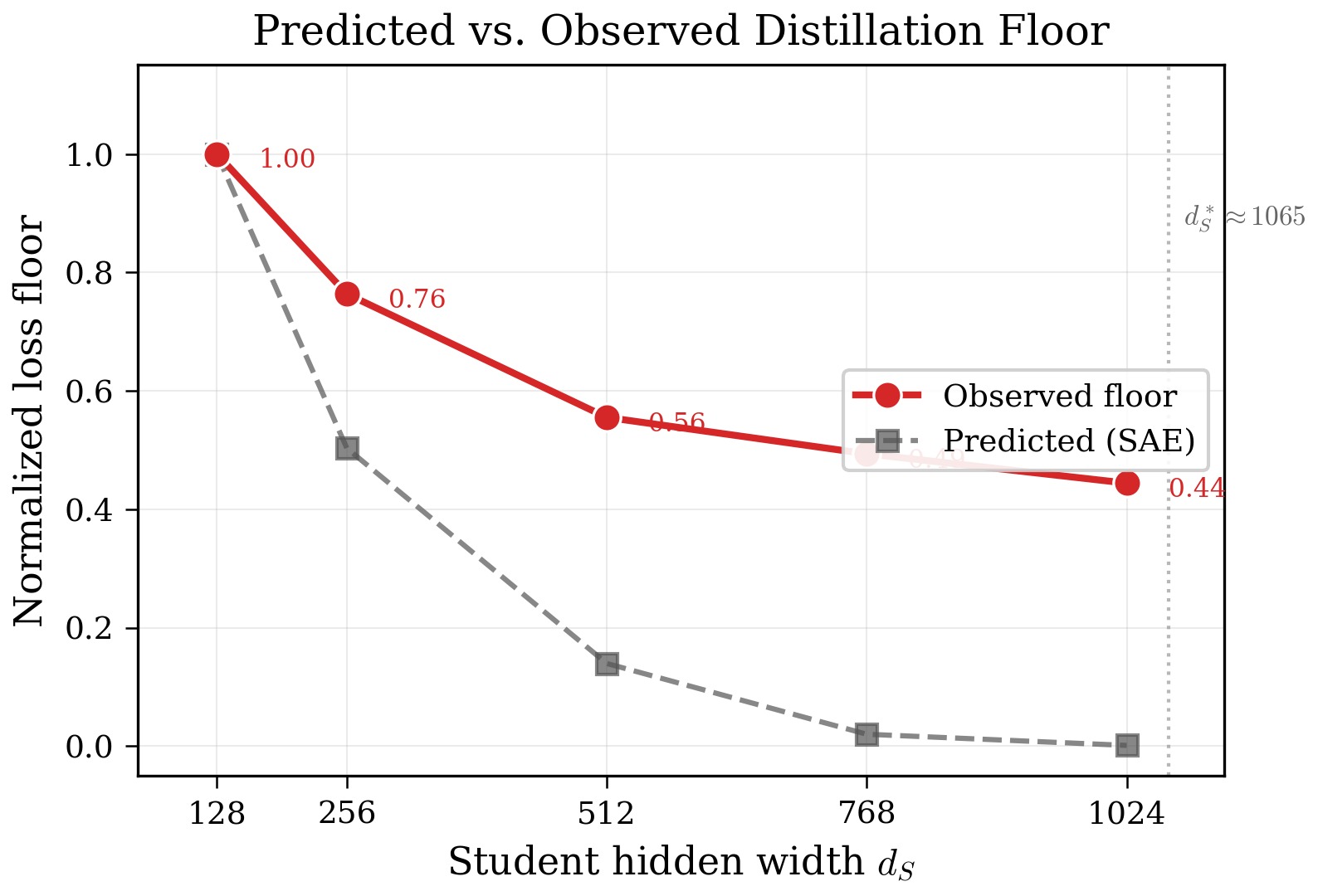}
\end{center}
\vspace{-0.5em}
\caption{Normalized predicted (SAE, dashed gray) vs.\ observed (distillation, solid red) floors. Both decrease monotonically; the widening gap reflects the constant baseline $B$ dominating at larger widths (see Table~\ref{tab:decomp}).}
\label{fig:pred_vs_obs}
\end{figure}

\begin{figure}[H]
\begin{center}
\includegraphics[width=\linewidth]{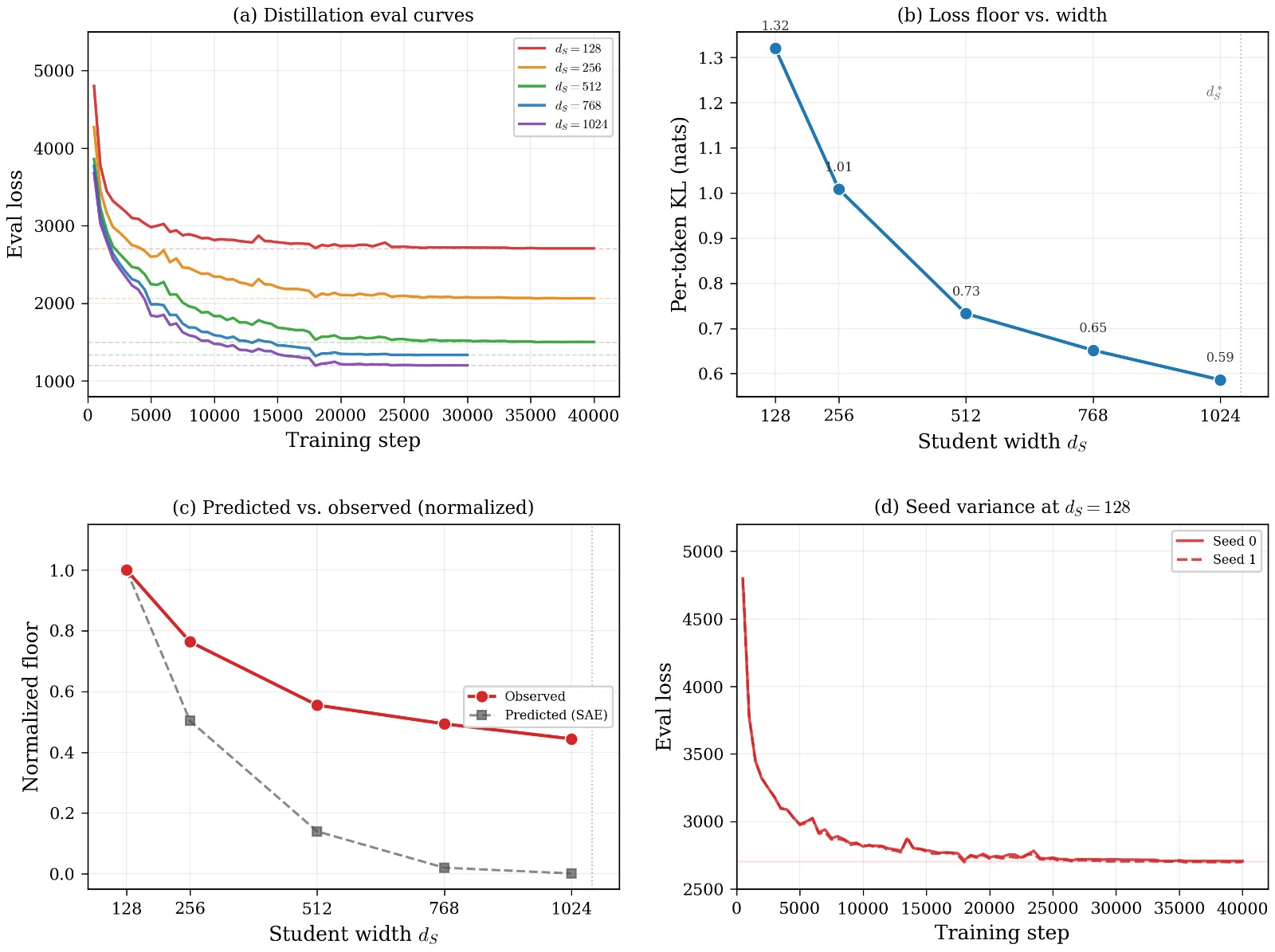}
\end{center}
\vspace{-0.5em}
\caption{Distillation summary. \textbf{Top left:} eval curves with floor estimates. \textbf{Top right:} per-token KL floor vs.\ width. \textbf{Bottom left:} normalized observed vs.\ predicted floors. \textbf{Bottom right:} seed variance at $\dS = 128$.}
\label{fig:distill_summary}
\end{figure}

\end{document}